\documentclass{myElsart}

\usepackage{amsmath}
\usepackage{amssymb}
\usepackage{bbm}
\usepackage{booktabs}
\usepackage{theorem}
\usepackage[]{natbib}
\usepackage{tikz}
\usetikzlibrary{arrows}
\usepackage{enumitem}

\usepackage{xr}
\usepackage{xr-hyper}

\newtheorem{theorem}{Theorem}[section]
\newtheorem{lemma}[theorem]{Lemma}
\newtheorem{proposition}[theorem]{Proposition}

\newenvironment{proof}[1][Proof]{\vspace{1em}\begin{trivlist}
\item[\hskip \labelsep {\bfseries #1}]}{\hfill $\Box$\end{trivlist}\vspace{1em}}

\begin{document}

\begin{frontmatter}

\title{High-dimensional networks and mean squared error for possibly misspecified models}
\author{Lourens Waldorp}\ead{waldorp@uva.nl} 
\address{University of Amsterdam, Nieuwe Achtergracht 129-B, 1018 NP, the Netherlands}

\begin{abstract}
To avoid missing important variables and their connections in networks, more and more variables are included in network analysis. Here we show that in a setting with many more parameters than observations (high-dimensional) it is possible to get a conservative (i.e., low false positive rate) estimate of the neighbourhood for each node (which connections are in the network). A neighbourhood is often estimated with a linear model, and this leads to two interesting cases: (i) If the true model is linear, then neighbourhood selection work reasonably well, and (ii) if the true model is nonlinear, then neighbourhood selection requires a penalty for the high dimensions. Here we show the impact of the ridge parameter on the mean squared error, and how this leads to low test variance and hence to neighbourhoods with large numbers of edges. We connect these insights with results from machine learning, where the so-called double descent (when more parameters are included than observations, the mean squared error goes down a second time) has put the traditional view on model selection upside down. Essentially, for adequate neighbourhood selection in models with a large number of parameters, the volume of the model space needs to be included in the penalty. Most neighbourhood selection methods (e.g., Lasso, AIC, BIC) lead to spurious edges (high false positive rate), but we prove that in the high-dimensional setting, minimum description length leads to correct neighbourhood selection or smaller (low false positive rates) in both cases when either the model is correctly or incorrectly assumed linear.  
\end{abstract}

\end{frontmatter}
\endNoHyper

\section{Introduction}\noindent
Often we are concerned that the number of observations in a study is insufficient to warrant proper analysis using networks \citep{Epskamp2018c}. In particular, a network contains many possible connections and therefore many parameters to be estimated. For instance, with 20 nodes, there are already 190 parameters for the edges in a network. In the classical statistical framework we would need  about 1000 observations to reliably estimate these parameters.

Recent advances have brought several estimation techniques that allow for the so-called high-dimensional situation, where we have more edges than observations (overparameterisation). Examples are the least absolute shrinkage and selection operator \citep[lasso,][]{Tibshirani:1996} and the ridge estimator \citep{Hoerl:1970,rao90}. The lasso is popular because it simultaneously estimates and selects edges to a node \citep{Waldorp:2024}. However, the lasso does not have a closed form solution and is discontinuous, which makes standard inference problematic \citep[although there are workarounds, see e.g.,][]{Buhlmann:2014b,Dezeure:2015}. In contrast, the ridge estimator does have a closed solution and is continuous. This makes the ridge estimate more amenable to analysis and allows for an analytical approach to determine why generalisation is still possible in the over-parameterised regime. The ridge estimate is constructed by adding independent dimensions to the predictors so that a unique solution can be obtained \citep{Hoerl:1970,rao90}. The solution is biased, but can relatively easily be used for inference \citep{Buhlmann:2013}. The ridge estimate can also be used for neighbourhood selection, i.e., using Akaike and Bayes information criteria and minimum description length to obtain a reliable estimate of the neighbours of each node in a network. Here we focus on reliable model selection for high-dimensional linear models used to estimate the neighbourhood of any node in the network, where the true edge relations are either linear or nonlinear.  

We use the well-known nodewise regression (neighbourhood selection) framework in graphical modelling \citep{Lauritzen96,Meinshausen:2006,Jankova:2017}. In nodewise regression, the procedure to find the nodes that are connected is completely local; each node is in turn used as the dependent variable in a regression on all remaining nodes. For each node, the objective is to identify the correct set of non-zero coefficients (neighbours) from the (large) set of remaining nodes. Figure \ref{fig:true-network} shows three graphs with different number of nodes ($15$, $32$ and $320$) from the perspective of one node; the objective is to identify five non-zero edges in the space of all other zero edges (sparse solution). When all nodes have been the focus of a regression, all neighbourhoods can be `stitched' together to form the entire graph.  For nodewise selection the objective is to select the correct neighbourhood or smaller but not larger (i.e., no false positive edges), because then, in stitching together the entire graph, the number of false positive edges will be high in large graphs. 
\begin{figure}
\begin{tabular}{@{\hspace{-1em}} c @{\hspace{1em}} c @{\hspace{1em}} c}
\textsf{\small 15 nodes}		&\textsf{\small 32 nodes}			&\textsf{\small 320 nodes}\\[0em]
\pgfimage[width=.33\textwidth]{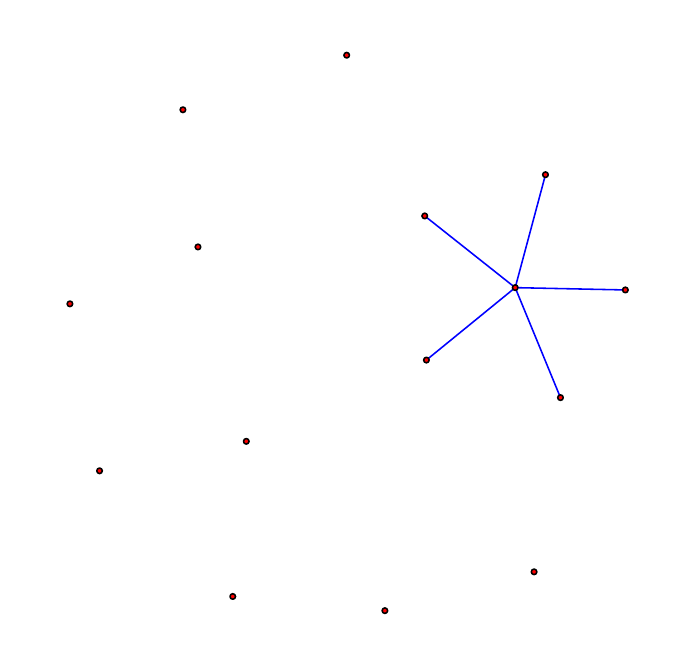}		&\pgfimage[width=.33\textwidth]{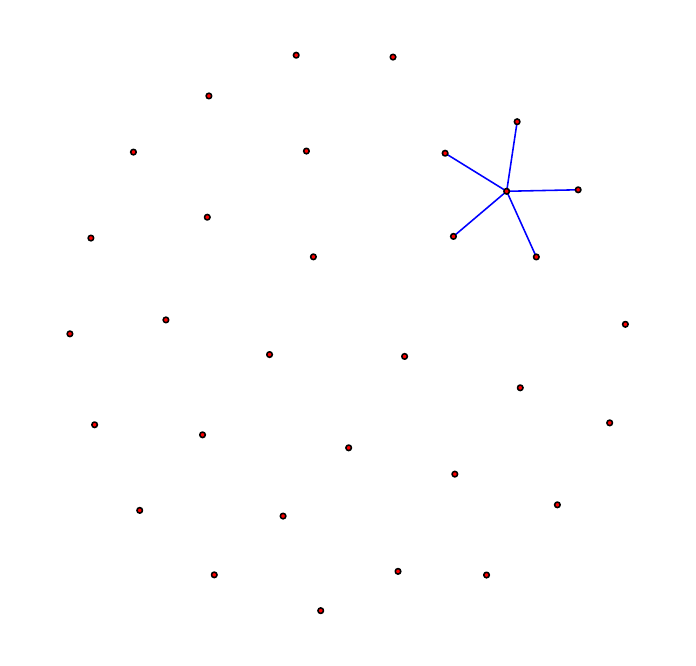}	&\pgfimage[width=.33\textwidth]{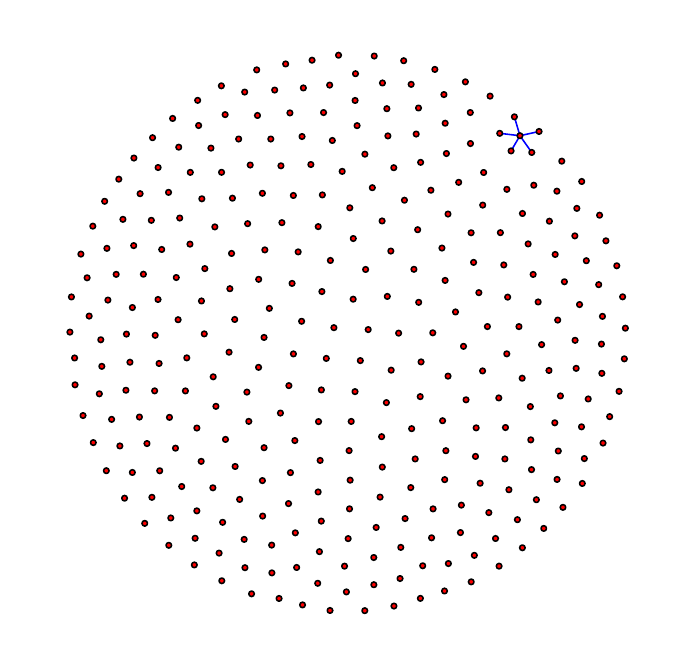}
\end{tabular}
\caption{ Graphs from the perspective of nodewise regression for graph sizes $15$, $32$ and $320$. For any node there are five non-zero coefficients (blue edges), so that each regression needs to identify among the remaining nodes the five non-zero coefficients.  }
\label{fig:true-network}
\end{figure}
%

To achieve this, we use recent work on the impact of estimation in the high-dimensional setting on the mean squared error (MSE) in regression. Both in \citet{Bartlett:2020} and \citet{Hastie:2022} the so-called ''ridgeless'' estimator was used (see Section \ref{sec:gaussian-graphical-model}) to conclude that, surprisingly, the MSE remains well behaved even in high dimensions. This implies that in principle it is possible to perform reliable neighbourhood selection in the high-dimensional setting. \citet{Cheema:2020} investigated the geometry of high-dimensional models relating it to the impact of using a large parameter space on model selection, concluding that when the large parameter space of the overparameterised model is properly taken into account, then model selection should work. \citet{Dwivedi:2020} adapted minimum description length, a model selection procedure, to be computationally efficient and reliable for large datasets. 
Other approaches, related to hypothesis testing, have also been investigated in the high-dimensional setting \citep[e.g.,][]{Meinshausen:2010,Buhlmann:2013,Shah:2023}. However, the selection is based on false positive control by thresholding, which is often resolved by additional assumptions, such as independence for the familywise error rate or false discovery rate \citep[e.g.,][]{DasGupta:2008}.

Our contribution builds on these works and mostly on the work by \citet{Cheema:2020}. Using linear neighbourhood selection, we leverage results about the impact of high-dimensional ridge estimation on MSE, and use this to explain why overfitting (many parameters in linear models) can unexpectedly result in low MSE, causing problems for nodewise selection. Penalties in the Akaike information criterion (AIC), for instance, are inadequate and mostly select too many edges (overfit). But we prove that the penalty of the minimum description length, as given in \citet{Grunwald:2007}, in the high-dimensional setting leads the correct neighbourhood or smaller, in both cases when the true model for edge relation is linear or nonlinear. 


We briefly explain the Gaussian graphical model in Section \ref{sec:gaussian-graphical-model} and discuss the problem of estimation in the high-dimensional setting. We describe that ridge regression leads to a possible choice in obtaining a unique estimate in high dimensions. Then in Section 
\ref{sec:mean-squared-error-bias} we discuss mean squared error (MSE). We show that the MSE remains low unexpectedly after the interpolation point (where the number of parameters equals the number of observations). This MSE behaviour is directly related to the ridge parameter, and we show that the ridge solution is directly related to constraining the parameter space. To achieve this, we discuss the MSE decomposition in Section \ref{sec:mse-decomposition} in terms of squared bias and variance, also in the case of model misspecification. Then in Section \ref{sec:model-selection} we turn to model selection. 
Here we describe minimum description length as an example where the geometry of the model space is taken into account as a penalty, and prove that this will lead to correct neighbourhood selection or smaller. In Section \ref{sec:simulations} we use simulations to confirm the theoretical results of neighbourhood selection.

\section{Estimation in the Gaussian graphical model}\label{sec:gaussian-graphical-model}\noindent
The Gaussian graphical model (GGM) is a Gaussian (normal) multivariate distribution associated with a set of nodes and edges in a network. A node $j$ in the network represents a Gaussian random variable $X_j$, and an edge $(i,j)$ between two nodes $i$ and $j$ in the network represents a conditional dependence relation between the two variables $X_i$ and $X_j$ given all other variables $X_k$ with $k\ne i,j$. 

The main characterisation in a GGM of conditional dependence is the partial correlation (or partial covariance). That is why a GGM is sometimes referred to as a partial correlation network. This characterisation makes sense because whenever a partial correlation is 0, then the two variables are conditionally independent \citep[][Proposition 5.2]{Lauritzen96} and so no edge is present in the network; when a partial correlation is $\ne 0$, an edge is present in the network. So, to construct a network we need to determine the 0 partial correlations. 
Obtaining the partial correlations can be done by estimating the inverse covariance covariance matrix by (penalised) maximum likelihood \citep{Friedman:2008,Jankova:2017}.

The partial correlations can also be determined from regressions. This is called neighbourhood selection or nodewise selection \citep{Meinshausen:2006,Buhlmann:2014b,Maathuis:2018,Waldorp:2021}. This is because the regression coefficient is proportional to the partial covariance. Hence, if the partial covariance is 0, then the regression coefficient must also be 0 (see Appendix \ref{sec:appendix-gaussian-graphical-model} for details). 
Thus, determining the GGM is reduced to a series of regressions, one regression for each node regressed on all $p-1$ remaining nodes, with the objective of determining the neighbourhood. Then, when all of the neighbourhoods have been determined, they can be put together to make up the entire graph. We therefore discuss a single regression, representing all $p$ regressions for the entire graph, selecting the nodes that are part of the neighbourhood of a node.

\subsection{Nodewise regression}\noindent
In nodewise selection, each node in turn is the dependent variable and the remaining nodes are the independent or predictor variables. We call variable $X_i$ the dependent variable and denote it by $Y$ and the remaining nodes are $X_j$ for $j\ne i$. In the scenarios we discuss we will be selecting a set of nodes that are relevant to predicting $Y$. We let $J$ denote the index set for the predictors used in predicting $Y$ (which equals $X_i$) with a subset of the remaining nodes. The set $J$ never contains the dependent variable $Y=X_i$ (no self loops), so that $J$ is a subset of $\{0,1,\ldots,i-1,i+1,\ldots,p\}$. We assume there is a true subset $S$, referred to as the support, such that $\beta_{j}\ne 0$ if $j\in S$ and $\beta_{j}=0$ if $j\notin S$. 
To obtain the estimates of $\beta_j$ from the linear regression, we minimise the least squares function 
\begin{align*}
\text{LS}(J)=\sum_{i=1}^n ( Y_i - \hat{Y}_J)^2\quad \text{ and } \quad \hat{Y}_J = \sum_{j\in J} X_j\hat{\beta}_j.
\end{align*}
For each model (subset) $J$ we can obtain an estimate of the coefficients corresponding to the variables $x_j$ such that $j\in J$. By estimating $\beta_j$ we obtain information on the neighbourhood of node $i$: for any $\beta_j\ne 0$ we draw the edge $(i,j)$; and if $\beta_j=0$, then there is no edge between $i$ and $j$. We therefore require estimates of the $\beta_j$ with $j\in J$.

We can rewrite the model in matrix algebra, with $Y=(Y_1,Y_2,\ldots,Y_n)$ the $n$-dimensional vector, $X=(X_1,X_2,\ldots,X_p)$ the $n\times p$-dimensional matrix and $\beta=(\beta_1,\beta_2,\ldots,\beta_p)$, $Y = X\beta$. The least squares function can be written in terms of the Euclidean distance function $|| a ||_2 = \sqrt{\smash[b] a_1^2 +a_2^2 +\cdots a_p^2}$. The least squares function is then $\text{LS}(J)=||Y-X\beta_{J}||_2^2$.
Minimisation is achieved from the normal equations 
\begin{align}
X^\top X \beta = X^\top Y.
\end{align}
Similar to numbers on the real line $\mathbb{R}$, we require that we can invert $X^\top X$ so that we obtain our estimate. For numbers on the real line we get for $xb=c$, that we can invert $x$ so that $b=c/x$. Similarly, the inverse $(X^\top X)^{-1}$ has the property that $(X^\top X)^{-1}(X^\top X)=I$, where the identity matrix $I$ is such that $I X=X$. Then our least squares (LS) estimate is 
\begin{align}\label{eq:least-squares-estimate}
\hat{\beta}^{LS}=(X^\top X)^{-1}X^\top Y.
\end{align}
It is clear from this derivation that we need the inverse $(X^\top X)^{-1}$. This assumes that the rank (i.e., the number of linearly independent vectors in $X$) is at least $p$. A necessary (but not sufficient) condition is that $p<n$. For sufficiency we also require that there is no collinearity (i.e., the eigenvalues of $X^\top X$ are all positive). Given these assumptions, in the linear model the estimate $\hat{\beta}^{LS}$ is unbiased (i.e. $\mathbb{E}(\hat{\beta}^{LS})=\beta_S$) and has smallest variance (i.e., $\text{var} (\hat{\beta}^{LS}) \le \text{var} (\beta_S)$ for any other estimate $\beta_S$).
Good references for these results are, for instance, \citet{rao90} and \citet{Seber:2012}.


\subsection{Estimation when $p>n$}\label{sec:appendix-estimation-p-n}\noindent
We now know that in order to obtain an estimate $\hat{\beta}$ we require the inverse of $X^\top X$ to exist. If $p>n$ then this is not the case. As a consequence, there is no unique solution \citep{PringleRayner71,Searle71,Schott97}. We can understand this by considering the eigenvalues (spectrum) of $X^\top X$, denoted by $\lambda_j$. If the $p\times p$ matrix $X^\top X$ has full rank $p$, then it has $p$ eigenvalues $\lambda_j>0$ \citep{Schott97}. If $X^\top X$ does not have full rank (because in our case $p>n$), but has rank $k$, say, then there are $p-k$ eigenvalues equal to 0. This refers to the idea that a part of the space spanned by $X^\top X$ projects on the null-space (i.e., the space such that $X^\top X u=0$ for any vector $u\ne 0$). Hence, there is no unique solution when $X^\top X$ is $p>n$.

One way to obtain a unique solution is to impose constraints on the vector $\beta$. One of the constraints is to impose an upper bound on the sum of absolute values, i.e., $\sum_j  |\beta_j| \le c$, for some $c>0$. Then the problem 
\begin{align*}
\min_{\beta} ||Y-X\beta ||_2^2 \quad \text{ such that }\quad \sum_{j\ne i} |\beta_j| \le c,
\end{align*}
is called the least absolute shrinkage and selection operator, abbreviated by lasso \citep{Tibshirani:1996}. The lasso has several attractive properties, like selecting the non-zero coefficients and consistency (i.e., converging with $n$ to the true value), given certain strong assumptions \citep[see, e.g.,][]{Wainwright:2009,Buhlmann:2011,Hastie:2015}. The lasso has no closed form solution and is not amenable to inference \citep{Potscher:2009,Buhlmann:2014b}. However, several workarounds have been obtained, like the debiased lasso \citep{Geer:2014,Waldorp:2024} and the multi sample-split \citep{Meinshausen:2009,Dezeure:2015}.

Another constraint that can be imposed on the vector $\beta$ is to upper bound the sum of squares of the parameters. This leads to the problem 
\begin{align*}
\min_{\beta} ||Y-X\beta ||_2^2 \quad \text{ such that }\quad \sum_{j\ne i} \beta_j^2 \le c.
\end{align*}
The solution to this minimisation problem (in terms of the Lagrangian) is given by
\begin{align}\label{eq:ridge-estimator}
\hat{\beta} = (X^\top X + \alpha I_p)^{-1} X^\top Y,
\end{align}
where $\alpha>0$ is some constant. This is called the ridge estimator \citep{Hoerl:1970}. It depends on setting the value $\alpha$. This solution is equivalent to jointly minimising the least squares function and the sum of squared parameters \citep{Rao:1999}. Hence we obtain the solution such that $||\hat{\beta}||_2^2=c$ \citep{Hastie:2001}. 
The solution $\lim_{\alpha\to 0}\hat{\beta}$ is called the minimum norm solution \citep{Ben-Israel:2003} or the ridgeless solution \citep{Hastie:2019,Dar:2021}. 
Note that if $n>p$, $\alpha=0$, and there is no multicollinearity, then the ridge estimator $\hat{\beta}$ in (\ref{eq:ridge-estimator}) reduces to the standard least squares estimator in (\ref{eq:least-squares-estimate}) where $X$ has full rank $p$. 

The benefits of ridge regression over the lasso are that (a) the ridge regression has a closed form solution, (b) the ridge estimator is continuous and so allows a sampling distribution and the calculation of standard errors, and (c) the ridge estimator is able to handle multicollinearity \citep{Buhlmann:2013}.
Property (b) is especially appealing in practice because it allows for direct inference with the ridge estimate \citep{Buhlmann:2013,Waldorp:2024}. The estimate has a bias $(X^\top X + \alpha I_p)^{-1} X^\top X-I$, which shows that if $\alpha\to 0$, and there is no collinearity, then the bias reduces to 0. 
The variance can be computed from which the standard errors are obtained for confidence intervals and $p$-values \citep{Hoerl:1970,Gruber90}. In \citet{Buhlmann:2013} a strategy is proposed to reduce the bias and obtain more reliable confidence intervals and $p$-values; this is also discussed in \citet{Waldorp:2024}. 

The parameter $\alpha$ in (\ref{eq:ridge-estimator}) can be obtained by $k$-fold cross-validation. 
The ridge estimator also has a Bayesian interpretation, where $\alpha$ denotes common variance for the normally distributed prior distribution with independent parameters \citep{Gruber90}, but $\alpha$ can also be related to the  $\eta$-generalised Bayesian approach, where $\alpha=\eta^{-1}$ \citep{Grunwald:2017}. 
The ridge estimator is biased (i.e., the expected value of the estimate is not the true estimator), like the Lasso, but the solution is in closed form and is continuous. Hence, inference on the ridge estimator directly is possible \citep[but see][for improvements]{Buhlmann:2013} and, as we will see shortly,  the ridge estimator leads to shrinkage useful for model selection and can be linked to a geometric interpretation in high-dimensional models. 

%

\subsection{Realistic goals for high-dimensional nodewise graphical models}\label{sec:nodewise-graphical-model}\noindent
In the context of graphical models, the objective is to find the correct neighbourhood (nodes that are directly connected to a node) of all nodes in the graph, since then we obtain the correct set of non-zero edges for the entire graph \citep{Meinshausen:2006}. The ideal is, therefore, that with high probability, for any node, the set of estimated and selected non-zero edges $\hat{S}$ is equal to the set of true edges $S$, i.e., $\mathbb{P}(\hat{S} = S)= 1-\epsilon$, with small $\epsilon>0$. However, in the high-dimensional setting this requires a strong and untestable assumption that the smallest value $\beta_j$ is larger than a value in the order of $|S|\sqrt{\smash[b]{\log(p)/n}}$ \citep[the beta-min condition,][]{Buhlmann:2013}, where $|S|$ denotes the cardinality of the true set of non-zero edges for the node in question. Foregoing this assumption may lead to the case where the set of true edges in the neighbourhood of a node is in the set of estimated edges with high probability, i.e., $\mathbb{P}(S\subset \hat{S})=1-\epsilon$ \citep{Buhlmann:2011}. Consequently, this may lead to an undesirably high false positive rate. If for each node we select only a few additional, spurious edges, then for the entire graph, the set of spurious edges will be proportional to the set of nodes. So, for large graphs, the set of spurious edges will be large. Therefore, a more realistic objective in graphical modelling is to obtain a set of neighbours such that for any node, the estimated set of non-zero edges $\hat{S}$ is in the true set of edges $S$, i.e., $\mathbb{P}(\hat{S}\subseteq S)=1-\epsilon$. 

In order to achieve the goal of low false positive rate in the high-dimensional setting, it is necessary that the large space of possible models is accounted for \citep{Foygel:2013,Dwivedi:2020,Cheema:2020}. Hence, a penalty for allowing such a large space of possible non-zero edges should be part of the neighbourhood selection. We show in Section \ref{sec:minimum-description-length} that because of the penalty in minimum description length (MDL), this selection method achieves the goal of obtaining the correct neighbourhood or smaller, but will not lead to a neighbourhood that is too large (no false positive edges). First, we show the impact of the high-dimensional setting on the mean squared error, which leads to the problems in neighbourhood selection.

\section{Mean-squared error and the bias-variance trade-off}\label{sec:mean-squared-error-bias}\noindent
Determining the Gaussian graphical model was reduced to a series of regressions (one for each node) when performing nodewise selection \citep{Meinshausen:2006}. Hence, model selection can be applied to each regression separately. We therefore apply ideas of mean-squared error to a single regression to illustrate the central concepts involved.

Model selection is usually considered as a trade-off between model fit and generalisation to other datasets. Rather than fitting exactly all datapoints of a particular dataset (overfitting), the aim is to find regularity that will allow the model to describe other datasets as well and generalise \citep{Myung:1998,Grunwald:2005,Claeskens:2008}. How close the model follows the dataset is related to bias and how well the model generalises to other datasets is related to the variance; the mean-squared error is the sum of the bias squared and the variance \citep[made precise in Section \ref{sec:mse-decomposition},][]{Hastie:2001}. We first discuss the intuition of the bias-variance trade-off and how this may fail in the high-dimensional setting, and then in Section \ref{sec:mse-decomposition} we discuss rigorous results about the bias-variance trade-off in ridge regression. 

%
\begin{figure}\centering
\begin{tabular}{@{\hspace{-1em}} c @{\hspace{1em}} c @{\hspace{-1em}} c}
\textsf{}		&\textsf{test MSE}			&\textsf{test bias-variance}\\[-2.5em]
		&\pgfimage[width=.5\textwidth]{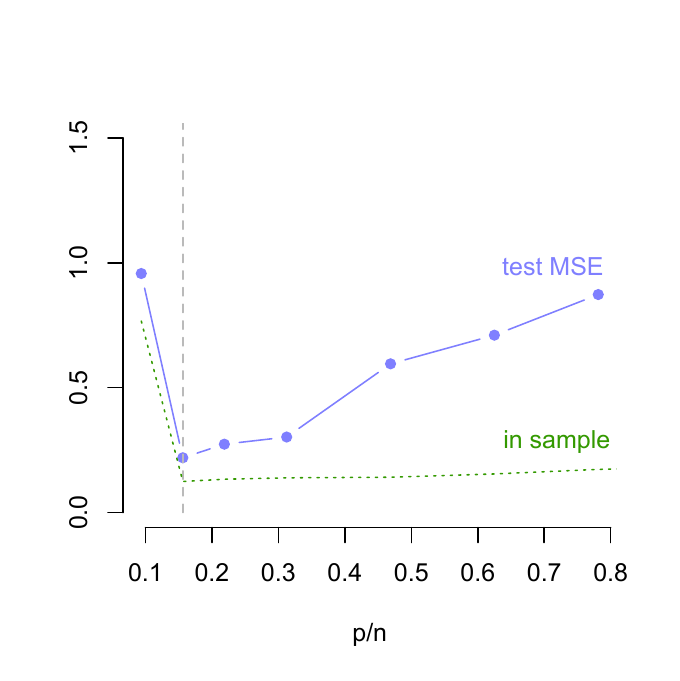}		&\pgfimage[width=.5\textwidth]{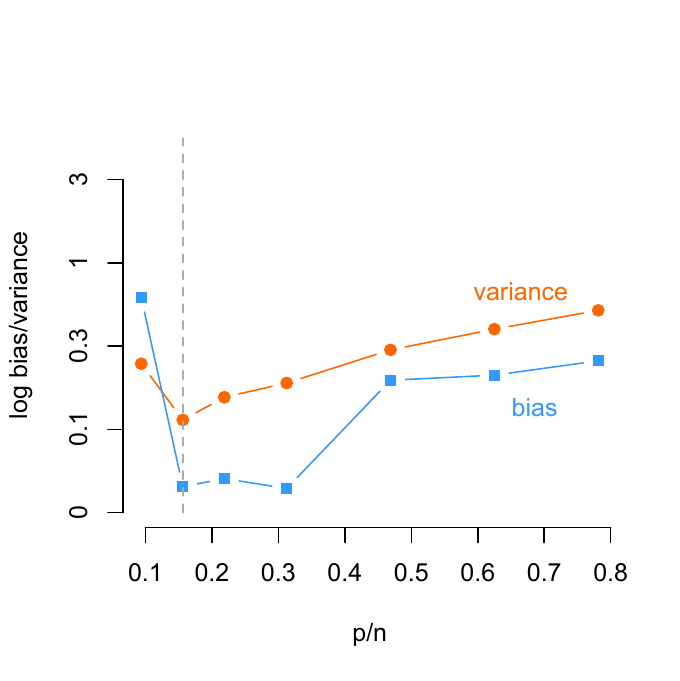}\\
	&	(a)		&(b)
\end{tabular}
\caption{In (a) the test and train mean squared error (MSE) for a true linear model as a function of $p/n$, the ratio of the number of parameters and the number of observations. The true model is at $p/n\approx 0.15$. In (b) are the squared bias (squares) and variance (circles) for the test set.}
\label{fig:mse-var-bias-classical}
\end{figure}
The classical idea of model selection is illustrated in Figure \ref{fig:mse-var-bias-classical}. The situation is a linear regression model as described in Section \ref{sec:gaussian-graphical-model} for the GGM. The regression coefficients are estimated by ridge regression for an increasing number of predictors that we put in the model with a subset of the data called the training set (we describe the details of the simulation in Section \ref{sec:simulations}). In Figure \ref{fig:mse-var-bias-classical}(b) we see the bias (squared)  of the test set (independent data not used for estimation) based on the coefficients of the training set as a function of $p/n$, the ratio of the number of parameters in the model and the number of observations (in the training set). Then the mean-squared error (MSE) is obtained in the remaining subset of observations, the test set; and is therefore called test MSE. The correct model is at $p/n\approx 0.15$ (the dashed vertical line). We see that the bias decreases sharply until the correct value is obtained and then increases again, and similarly for the variance. The squared bias and variance together make up the test MSE (also known as excess risk and prediction error). This is shown in Figure \ref{fig:mse-var-bias-classical}(a). Here we see that at the point $\approx 0.15$ (the correct model), where the test MSE is lowest, the corresponding  model should be selected. 
This is the classic case discussed in for instance \citet[][Chapter 7]{Hastie:2001} and \citet[][Chapter 1 and Chapter 2]{Grunwald:2005}. The test MSE is used instead of the MSE based on the training (estimation) data (seen as the dashed green line in Figure \ref{fig:mse-var-bias-classical}(a), in sample) because the training MSE is overly optimistic and so remains low  \citep[see][]{Efron:1986b,Hastie:2001}. This trade-off between bias and variance implies that the complexity of the model  cannot be too high because then the model tends to overfit, i.e., fit only the training data and does not generalise well to other data.

In recent years it has appeared that this classical scenario is not all there is to it. There are situations where a model can over fit (select too many non-zero coefficients) and still generalise well. This phenomenon \citep[sometimes called benign overfitting,][]{Bartlett:2020} is seen in Figure \ref{fig:ridge-mse-var-bias}(a). The linear regression coefficients are estimated using the ridge estimator where the number of parameters (predictors) is increased such that the ratio $p/n$ (for the training set) ranges between 0.1 to 2. At $p/n=0.1$ we have 10 observations per parameter and at $p/n=2$ we have 0.5 observations per parameter. The classical case is seen up to $p/n=1$ (i.e., $p=n$), the interpolation point (as in Figure \ref{fig:mse-var-bias-classical}). From the interpolation point on, the model is overparameterised ($p>n$) and has very low training MSE (green dashed line in Figure \ref{fig:ridge-mse-var-bias}(a)).
In Figure \ref{fig:ridge-mse-var-bias}(b), we see that, as expected, the bias decreases after the interpolation point. Unexpectedly, the variance starts to decrease as well, suggesting that the model will generalise well to other datasets. But this seems paradoxical since overparametrisation (and hence overfitting) would suggest that noise is being modeled. But the decreasing variance implies that the overparameterised model is capable of capturing regularities in other datasets, i.e., it will generalise well. 
\begin{figure}\centering
\begin{tabular}{@{\hspace{-1em}} c @{\hspace{1em}} c @{\hspace{-1em}} c}
\textsf{}		&\textsf{MSE}			&\textsf{bias-variance}\\[-2.5em]
	&\pgfimage[width=.5\textwidth]{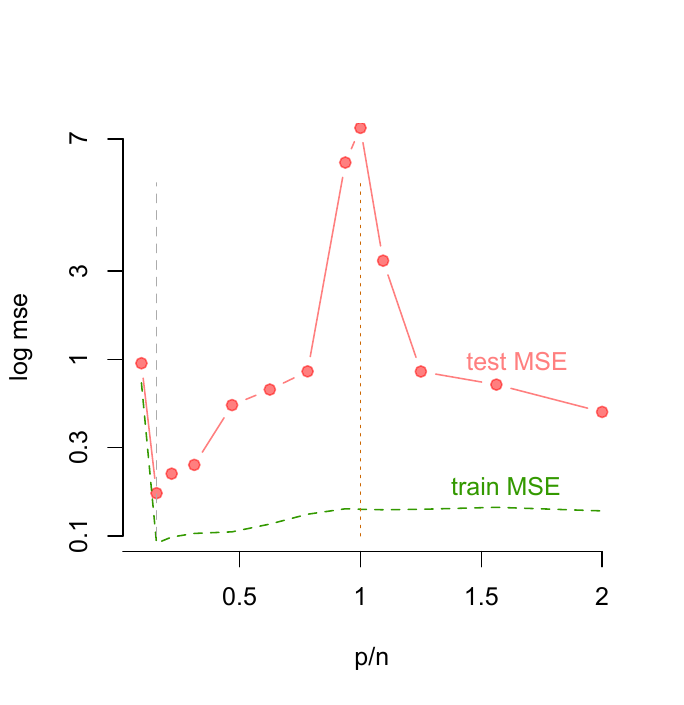}		&\pgfimage[width=.5\textwidth]{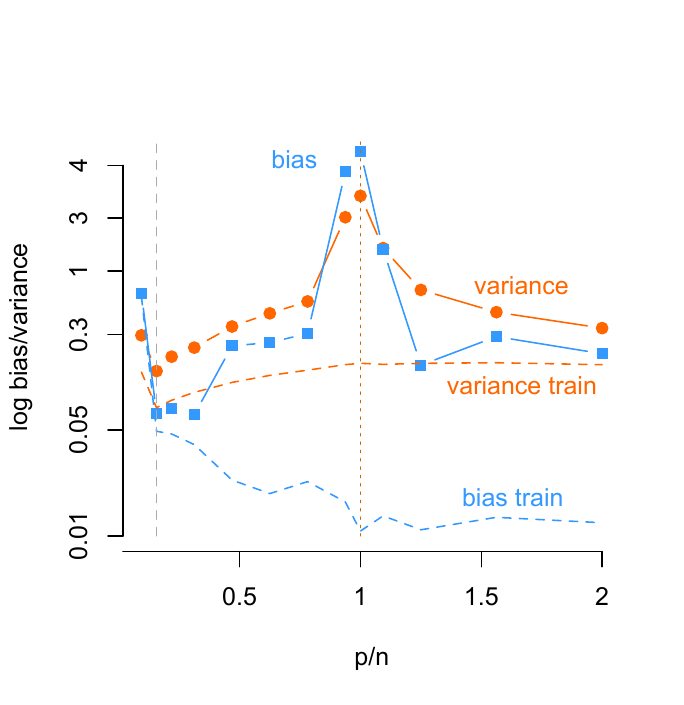}\\[-1em]
		&(a)		&(b)
\end{tabular}
\caption{In (a) the test and train mean squared error (MSE, in logarithm) for a true linear model as a function of $p/n$, the ratio of the number of parameters and the number of observations (in the training set). The true model with 5 predictors is at $p/n\approx 0.15$, indicated by the vertical grey dashed line. In (b) is the bias (squares) and variance (circles), both in logarithm, for the test set and for the training set (dashed lines). }
\label{fig:ridge-mse-var-bias}
\end{figure}

This phenomenon is often referred to as the double descent \citep{Belkin:2019}, since there is a second descent after the interpolation point ($p=n$). 
The double descent has been investigated thoroughly \citep[e.g.,][]{Hastie:2019,Bartlett:2020,Bartlett:2021} with some results as follows. We assume for these results that all random variables are normally distributed and have independent identically distributed predictors. 
\begin{itemize}
\item[(a)] In linear regression when the model is approximately correct and $p/n$ is large (overparameterised), the global minimum of the test MSE is obtained at the correct model. 

\item[(b)] If the model is misspecified and $p/n$ is large (overparameterised), then the global minimum of the test MSE \textit{could} be obtained at a largely overparameterised (incorrect) model.
\end{itemize}
The first result (a) implies that we are still able to correctly identify the correct neighbourhood for each node when approximately correctly specified. And the second result (b) implies that when the model is not correctly specified, e.g., nonlinear, it might be the case that in neighbourhood selection far too many edges will be selected because the MSE is lowest in that case. This last result explains why in deep learning (and in machine learning) astounding results have been obtained by hugely overparameterised models. Examples are object and speech recognition and traffic sign classification \citep[see, e.g.,][]{Goodfellow:2016}. 

The implications for neighbourhood selection are that obtaining the correct neighbourhood or smaller requires increasing the ridge parameter strongly, so as to counteract the small MSE that is obtained for large misspecified models. It is in general questionable whether this can be achieved by carefully selecting a ridge parameter (by e.g., cross-validation), or whether a separate penalty is required that counteracts the low MSE obtained in the case of high-dimensions. We will see in Section \ref{sec:minimum-description-length} that we require a separate penalty in the high-dimensional case.

\section{Mean-squared error decomposition in bias and variance}\label{sec:mse-decomposition}\noindent
We remain within the framework of nodewise selection and focus on a neighbourhood of any node to obtain the non-zero edges, which leads to regressing one node on the remaining nodes. 

We define the mean squared error (MSE) for linear models as a quantity that represents the expected loss incurred by estimating with model $J$ instead of the true model $S$. 
Let $\hat{Y}_{S}$ be the prediction based on the true set $S$ of predictors, so that $\beta_{j}\ne 0$ if $j \in S$ and $\beta_{j}=0$ if $j\notin S$, which we refer to as $\beta_S$. Furthermore, we assume that the data are generated according to an additive model where errors $e$ are added to $\hat{Y}_S$, so that 
\begin{align}\label{eq:linear-model}
Y = \hat{Y}_S + e = \sum_{j\in S} X_j\beta_j +e,
\end{align}
where the errors $e$ are independent and identically distributed with mean 0 and variance $\sigma^2$.

Test MSE (excess risk) is defined in terms of a so-called test set, an independent point (or set of points) not encountered before in the so-called training set \citep{Hastie:2001}. Let $(X_0,Y_0)\in\mathbb{R}^{p+1}$ be a test data point independent of, but from the same distribution as, all other $(X_i,Y_i)\in\mathbb{R}^{p+1}$. With the training data the ridge estimate $\hat{\beta}_J$ with model $J$ is obtained.  Then the excess risk \citep[out of sample prediction risk,][]{Hastie:2019} with respect to the true parameter $\beta_S$, is defined as the squared difference between the predictions obtained with model $J$ and the predictions obtained with the true model $S$ for test data $(X_0,Y_0)$ 
\begin{align}
R(J)=\mathbb{E}\left[ (Y_0-X_0^\top\hat{\beta}_{J} )^2 - (Y_0 - X_0^\top \beta_S)^2\right],
\end{align}
where the expectation is conditional on the training data $X$, which functions in the ridge estimate $\hat{\beta}_J = (X^\top X + \alpha I_p)^{-1} X^\top Y$ from (\ref{eq:ridge-estimator}). Similar to \citet{Hastie:2019} and \citet{Bartlett:2021} we can obtain the excess risk in terms of the squared bias and variance (see Lemma \ref{lem:mse-decomposition} in Appendix \ref{sec:appendix-mean-squared-error})
\begin{align}\label{eq:risk}
R(J) 
&= B(J) +V(J),
\end{align}
%
with
\begin{align}\label{eq:bias-variance}
B(J) = (\beta_S -\mathbb{E}\hat{\beta}_J)^\top \Sigma (\beta_S -\mathbb{E}\hat{\beta}_J),
\hspace{0.5em}
V(J) =  \mathbb{E}(\hat{\beta}_J-\mathbb{E}\hat{\beta}_J)^\top \Sigma (\hat{\beta}_J-\mathbb{E}\hat{\beta}_J),
\end{align}
where $B(J)$ is called the squared bias and $V(J)$ is called the variance, and $\Sigma  =\mathbb{E}(X_0 X_0^\top)$, which is the same as the covariance matrix of the predictors $\mathbb{E}(X^\top X)$. This is the famous variance and squared bias decomposition \citep{Hastie:2001}. The  classical idea is that reducing bias will increase variance and vice versa, thus balancing the MSE. Hence, a good model will minimise the MSE (excess risk) and most model selection methods are geared towards this goal.

The training bias $B(J)$ is 0 for the linear model, when $p<n$, and the linear model is correct \citep{Hastie:2019}. However, for the ridge estimator $\hat{\beta}$ in (\ref{eq:ridge-estimator}) the bias is non-zero \citep{Hoerl:1970}. It turns out that when the in-sample (training) residuals are 0 (i.e., overparameterised, see Appendix \ref{sec:app-residuals-training-test}), then still the variance on the test set need not be large \citep[][]{Bartlett:2021}. Here we explain how this can be achieved with ridge regression.

To show the impact of ridge regression on the variance and the squared bias (and hence the MSE), we will rewrite the variance and squared bias from (\ref{eq:bias-variance}). To do this we further assume that $\Sigma=\sigma_\xi^2 I$ (i.e., is isotropic). We then obtain the test variance (see Lemma \ref{lem:variance-eigen-values})
\begin{align*}
V(J) = \sigma_\xi^2\sigma^2_e \sum_{j=1}^n \frac{\lambda_j} {(\lambda_j + \alpha)^2},
\end{align*}
where $\lambda_j$ are the eigenvalues of $X^\top X$. Hence, by increasing $\alpha$ we see that we are lowering the variance $V(J)$, and hence the test MSE. Additionally, we can observe that if the contribution of the eigenvalues $\lambda_{d+1}$ up to $\lambda_{p}$, for some $d<p$, is small, then the variance will not increase, and hence, prediction may still be accurate. This implies that adding more predictors (making $p$ larger) will not necessarily inflate the variance $V(J)$ much, which is stabilised by the ridge parameter $\alpha$. This argument is used in \citet{Bartlett:2020,Bartlett:2021} to explain why in some cases machine learning with a large number of parameters predicts well; the noise is distributed among the large number of directions (given by the eigenvectors) so that none of the eigenvalues $\lambda_k$ with $k>d$ will be larger than the ridge parameter $\alpha$.

Similar to the test variance $V(J)$ the bias $B(J)$ can also be reduced to a more amenable form when we assume that $\Sigma=\sigma_\xi^2 I$. With this assumption we obtain (see Lemma \ref{lem:bias-eigen-values})
\begin{align}
B(J) = \sigma^2_\xi\frac{\gamma_1\alpha^2}{(\lambda_1+\alpha)^2}+ \sigma^2_\xi\sum_{j=2}^n\frac{\alpha^2}{(\lambda_j+\alpha)^2},
\end{align}
where $\gamma_1$ is the single non-zero eigenvalue of $\beta_S\beta_S^\top$. This shows different behaviour for the squared test bias than the test variance $V(J)$: the squared test bias could increase by increasing $\alpha$. 

This analysis shows that the peak in the test MSE for the ridge estimator in Figure \ref{fig:ridge-mse-var-bias}(a) is in some sense an artefact of choosing the wrong ridge parameter $\alpha=0.0001$ (used in Figure \ref{fig:ridge-mse-var-bias}). The peak disappears when we use a carefully selected ridge parameter $\alpha$ using $k$-fold cross-validation. This is shown in Figure \ref{fig:mse-var-bias} in Appendix \ref{sec:appendix-simulation-details}. It can also be seen in Figure \ref{fig:mse-var-bias} that the lasso also does not show such a peak but the lasso has generally larger MSE than the ridge.

We can now conclude from this analysis that there are two views on how to manage the test variance in the case of high-dimensional data. 
\begin{itemize}
\item[($a$)] We can increase the ridge parameter $\alpha$ directly, reducing the test variance $V(J)$, and hence obtaining reasonable test MSE. 
\item[($b$)] we can constrain the size of $||\beta ||_2$ in the ridge estimator, therefore, decreasing the model complexity (volume) of the model. 
We do this by constraining the ridge estimator with small $c$ such that $||\beta||_2\le c$. 
\end{itemize}
In both cases ($a$) and ($b$) we are either directly or indirectly constraining the total signal $||\beta||_2$. In model selection this can be done by either increasing $\alpha$, equivalently, decreasing $c$ in the ridge constraint $||\beta||_2\le c$, or by incorporating the complexity (volume) of the model. (The ridge parameter $\alpha$ can also be interpreted in terms of the signal-to-noise (SNR) with similar conclusions relating to SNR, see Appendix \ref{sec:appendix-geometry-metric-entropy}.)

\subsection{Model misspecification}\label{sec:model-misspecification}\noindent
More generally, a model is defined as a class of probability distributions (in a favourable case in terms of corresponding densities), and the model is correctly specified if the distribution is in the particular class used for optimisation \citep{Myung:1998}.  In our setting of nodewise regression with Gaussian variables, to determine the neighbourhood of a node, the model class is the set of univariate normal distributions parameterised by the mean $\mu$ and variance $\sigma^2$, with the conditional mean $\mathbb{E}(Y\mid x)$ given by a linear function $x\mapsto x^\top \beta$, so that the model is correctly specified. Model misspecification can then be defined in terms of the model class not containing the correct distribution \citep{Grunwald:2007,Dar:2021,Kleijn:2026}. Continuing our example of normal distributions, the conditional mean could be given by some nonlinear function $f$, say a logistic function $x\mapsto 1/(1+\exp(x^\top\beta))$, while we only consider linear models $x^\top\beta$ for estimation. One interpretation of the linear parameters when the true model is nonlinear, is in terms of a weighted average of derivatives of the conditional mean with respect to the covariate $x_j$, which reduces to mixtures of treatments in causal analysis \citep{White80,Kunievsky:2025}. The problem with misspecifying the conditional mean may lead to the situation where the residual variance is heterogeneous, and the `best' model in terms of being closest to the true model in Kullback-Leibler divergence, is a mixture of models which is not in the model space (non-convex model space). For more information on this see Appendix \ref{sec:appendix-misspecification} and the seminal work of \citet{Grunwald:2017}. 

For nodewise regression with misspecification, we assume that the model is some nonlinear function $f$ modelled by a linear function $X^\top\beta$. 
If the misspecification is characterised by an additive term, so that the model is $X^\top \beta + \xi$, where we do not know about $\xi$, then the misspecification can be described by the mismatch $\mathbb{E}||\xi||^2$ \citep{Hastie:2019,Dar:2021}. In our approach we do not assume a linear decoupling, but simply that there is some nonlinear function $f$.
In that case we obtain an additional term in the excess risk (see Lemma \ref{lem:mse-misspecification} in Appendix \ref{sec:appendix-mean-squared-error})
\begin{align}\label{eq:mse-misspecified}
R_f(J) &= B_f(J) + V_f(J) + M_f(J),
\end{align}
where
\begin{align*}
B_f(J) &=\mathbb{E}\left[ (f(X_0,\beta_S) - \mathbb{E}(X_0^\top\hat{\beta}_J))^2 \right],\\
V_f(J) &=\mathbb{E}\left[ (X_0^\top\hat{\beta}_J - \mathbb{E}(X_0^\top\hat{\beta}_J))^2 \right], \quad\text{and}\\
M_f(J) &= 2\mathbb{E}\left[ (Y_0-f(X_0,\beta_S))(f(X_0,\beta_S)-X_0^\top\hat{\beta}_J) \right].
\end{align*}
If we are able to bound the misspecification $|f(X_0,\beta_S)-X_0^\top\hat{\beta}_J|=O_p(K)$, we obtain a bound on the misspecification term by the Cauchy-Schwarz inequality: $M_f(J)\le \sigma O(K)$ (see Proposition \ref{prop:misspecification-bound}). This bound on the misspecification will lead to the possibility of obtaining the correct neighbourhood or smaller with minimum description length in the high-dimensional setting. Additionally, in linear models we can leverage the result that for $L\subseteq J$, we have that $\mathbb{E}[(Y_0- X_0^\top\hat{\beta}_J)^2] \le  \mathbb{E}[(Y_0- X_0^\top\hat{\beta}_L)^2]$ (see also Appendix \ref{sec:appendix-mdl-underfit}). Hence, the misspecification decreases as the ratio $p/n$ increases.

\section{Nodewise selection in high dimensions}\label{sec:model-selection}\noindent
A regression model is represented by a distribution for the residuals \citep{Hastie:2001} and, for the GGM using linear regressions, is identified with a set of nodes in the neighbourhood of a node, indexed by the set $J$. Increasing the number of included nodes in a nodewise regression is hence identified with different distributions, which can then also be indexed by $J$. Nodewise selection is therefore concerned with distinguishing probability distributions. The Kullback-Leibler divergence \citep[KL,][]{Cover:2006} is often used to distinguish distributions, where a KL of 0 indicates no difference between distributions, and a value larger than 0 indicates different distributions (see Appendix \ref{sec:appendix-akaike-information-criterion}). Both the Akaike information criterion (AIC) and minimum description length (MDL) can be considered as minimising the KL of a hypothesised distribution (our linear regression models) and some generative  (unknown) model \citep{Grunwald:2000,Myung:2006,Claeskens:2008}. The AIC estimates the KL and corrects the incurred bias of doing so (see Appendix \ref{sec:appendix-akaike-information-criterion}). The MDL, on the other hand, applies the idea of encoding the data and hypothesised model to obtain a reasonably approximating model, which we explain here.


In the high-dimensional case nodewise selection is problematic precisely because of the large parameter space. In high dimensions we obtain eventually near-zero training error, even with a linear model \citep[][see Appendix \ref{sec:app-residuals-training-test}]{Bartlett:2020}. This is because the parameter space of the model is so vast (i.e., high complexity), that nearly all data points can be accounted for. Conceptually, the volume of the parameter space more or less explodes with increasing dimension and, hence, dominates the regression (see Appendix \ref{sec:appendix-high-dimensional-space} and Appendix \ref{sec:appendix-geometry-metric-entropy}). This may lead to selecting too many neighbours (over fit), which, in nodewise selection leads to a graph with many spurious edges. In order to obtain a correct neighbourhood or smaller in such high-dimensional situations, this increased model complexity needs to be taken into account explicitly, especially in the case of misspecification. 



Here we discuss the MDL in some detail and show that in the high-dimensional setting it leads to no false positive edges with high probability. The reason is that the penalty (model encoding) counteracts the large number of possible edges in a neighbourhood. For comparison, we also show results for the AIC, BIC, the Lasso, and another version of MDL for nodewise selection. We discuss these in Appendices \ref{sec:appendix-akaike-information-criterion} for the AIC and \ref{sec:appendix-minimum-description-length-optimised} an optimised version of MDL. There are excellent books that describe these derivations well. For instance, the AIC and BIC (and other model selection criteria) are described in \citet{Claeskens:2008} and MDL is described extensively in the book by \citet{Grunwald:2007}.

\subsection{Minimum description length}\label{sec:minimum-description-length}\noindent
The minimum description length (MDL) is derived from the idea that compression (encoding) of data can be translated into probabilities \citep{Grunwald:2005}. MDL is concerned with the length of (prefix or invertible) optimal codes which can be separated into a part where the data are encoded given the model and a part where the model is encoded.  The length of the data can be encoded by the log likelihood  \citep{Hansen:2001}, i.e. $\text{length}(Y,x)=-\log f(Y\mid x,\hat{\beta})$, where $f$ is the conditional density of $Y$ given the predictors $X=x$ (see Appendix \ref{sec:appendix-akaike-information-criterion}). Additionally, the model is encoded, called the complexity, which leads to a specific formulation depending on the type of encoding \citep{Grunwald:2007}. This encoding term is the penalty for the possibly large parameter space and will to a certain extent guard against overfitting. 

For normally distributed random variables the MDL can be approximated by \citep{Myung:2006,Grunwald:2007}
\begin{align}\label{eq:mdl-full}
\text{MDL}(J) = \underbrace{\frac{n}{2}\log \hat{\sigma}^2_J}_{\text{data  code length}}+ \underbrace{\frac{1}{2}\text{tr}(S_J)\log n + \log \int_{B} \sqrt{\det(F_J)}d\beta}_{\text{model code length}} + o(1),
\end{align}
where $\text{tr}(S_J)$ is the trace of matrix $S_J$ representing the effective number of parameters for model $J$, and $F_J$ is the Fisher information matrix for model $J$, which is $F_J=X^\top X/\sigma^2$ when $n>p$,  
and the true model is exponential family \citep[][Chapter 2]{Grunwald:2005}.
The matrix $S_J$ is used to determine the number of effective parameters, and is $S_J=X(X^\top X)^{-1} X^\top$ when $n>p$; the trace of this matrix is $p$ (this is the same as the rank in this case). 

However, for the high-dimensional case ($p>n$) we find that $X^\top X$ is singular, implying that $\det(X^\top X) =0$. This is because the rank of $X^\top X$ is $\min\{n,p\}$, and because $n<p$, we have that the rank in this case is $n$, while the dimensions are $p\times p$. Thus, we must have that there are $p-n$ dimensions that are not represented by $X^\top X$ (i.e., $p-n$ vectors project onto the kernel or null space). We can use ridge regression to alleviate the problem by using instead of $X^\top X$ the matrix from ridge regression $W(\alpha)=X^\top X + \alpha I$ \citep{Hastie:2001}. We then find that $S_J=X(X^\top X + \alpha I)^{-1} X^\top$. This gives the general equation for the number of effective parameters \citep[][Section 7.6]{Hastie:2001}
\begin{align*}
df(\alpha) = \text{tr}(S_J)= \text{tr} X(X^\top X + \alpha I)^{-1} X^\top= \sum_{j=1}^{\min\{p,n\}} \frac{\lambda_j}{\lambda_j+\alpha},
\end{align*}
where $\lambda_j$ are the eigenvalues of $X^\top X$.

To approximate the integral in the penalty encoding the model in (\ref{eq:mdl-full}), we make use of the fact that in linear regression the Fisher information is independent of the parameter $\beta_j$, and hence the integral represents the open $p$-dimensional ball $\mathbb{B}^p(||\hat{\beta}||_2)$ over the Fisher information where we used the ridge parameter $\alpha$ (see Lemma \ref{lem:volume-linear-model} in Appendix \ref{sec:appendix-geometry-metric-entropy}). The $p$-dimensional ball is centered at $||\hat{\beta}||_2$ because in the ridge solution we obtain the solution such that $\alpha(||\beta||_2^2-c)=0$.  We then obtain  \citep{Cheema:2020,Grunwald:2007}
\begin{align}\label{eq:mdl-penalty}
\log \int_{B} \sqrt{\det(F)}d\beta \approx \log \det(X^\top X+\alpha I) + \log \mathbb{B}^p(||\beta||_2)/\sigma^p.
\end{align}
The $p$-dimensional ball in the last term, $\log \mathbb{B}^p(||\beta||_2)$, represents the volume of the parameter space and is discussed in Appendix \ref{sec:appendix-high-dimensional-space}. The $\log$ of this term will become large, and hence apply a stronger penalty with increasing dimension. This is why increasing the dimension can have positive effects on the generalisation error (variance in the test set). 
However, in determining the volume in this way, we have assumed that the linear model is correct. Hence, we may expect that when the model is misspecified, the term in the integral encoding the model in (\ref{eq:mdl-penalty}) may not be sufficient to guard against incorrect decisions. \citet{Grunwald:2017} show that when the model is misspecified the model space can be non-convex, and the best approximation to the true model may not be in the model space (see Appendix \ref{sec:appendix-misspecification}). However, using linear models we can with high probability conclude that MDL as defined in (\ref{eq:mdl-full}) with (\ref{eq:mdl-penalty}) will be able to guard against overfitting. This is because of the penalty term $\mathbb{B}^p(||\beta||_2)$, as predicted by \citet{Cheema:2020}. We are therefore able to achieve the goal mentioned in Section \ref{sec:nodewise-graphical-model}, that the selected neighbourhood of edges in a graph $\hat{S}$ with MDL is in the true neighbourhood $S$, with $d=|S|$. 
\begin{proposition}\label{prop:mdl-underfit}
Let $S$ denote the true support for a (possibly nonlinear) data generating process of $Y_i$ for $i=1,\ldots,n$ with $n$ fixed, and let $J$ represent the set of predictors such that $S\subset J$, where $|S|=d$ and $|J|=p$. We test linear models $X\beta_J$ for any $J$ (including $S$) by MDL as given in (\ref{eq:mdl-full}) with (\ref{eq:mdl-penalty}). Let $p$ increase and $p > d$. If $\mathbb{E}\hat{\sigma}^2_J>\gamma$, for some $\gamma>0$, then with probability at least $1-\tfrac{d}{p}\exp(-(p-d))$ it holds that $\hat{S}\subseteq S$; equivalently, the probability of false positives goes to 0. 
\end{proposition}
A proof can be found in Appendix \ref{sec:appendix-mdl-underfit}.


We refer to MDL in (\ref{eq:mdl-full}) with (\ref{eq:mdl-penalty}) as MDL. The complexity of MDL has also been used without the integral term \citep{Rissanen:1986,Grunwald:2000}. This approximation is in that case equivalent to the Bayesian information criterion (BIC) introduced by \citet{Schwartz78}. Hence, we call that version MDL-S. 

\section{Simulations}\label{sec:simulations}\noindent
In order to determine the behaviour of nodewise selection in realistic high-dimensional scenarios in graphical models, we perform simulations for a single neighbourhood. This corresponds to the graphs shown in Figure \ref{fig:true-network}, with the perspective of a single node for which we need to determine the connected nodes (neighbourhood) among the nodes that are not connected, of which there are many more. 

The data are generated by a Gaussian graphical model. An inverse covariance matrix was generated where each node had 5 neighbours with edge coefficient 1. We then selected a node at random. Data were generated where the total number of observations was fixed at $n=40$. We vary the graph size so that the number of possible edges to be estimated ranges from $p= 3$ to $p=320$ (see Figure \ref{fig:true-network}). So, we obtain a ratio of the number of edges and the number of training observations of $p/n=0.1$ to $2$ with $p=64$ and $p/n=10$ for $p=320$. Edges (conditional dependences) are generated either according to a linear model (correct model specification) or a sigmoid function (model misspecification). The regression coefficients for nodewise selection are estimated as a linear model by ridge regression (\ref{eq:ridge-estimator}) with a training set of $n_\text{train}=32$ ($80\%$ of the total sample size $n=40$).  The test set is of size $n_\text{test}=8$. We consider the case where we know the true model is linear and we estimate the linear model (case (a) in Section \ref{sec:mean-squared-error-bias}), and we misspecify the model because the true model is sigmoidal and we use a linear approximation (case (b) in Section \ref{sec:mean-squared-error-bias}). We compare the AIC, AIC-CV (where $\alpha$ is obtained with cross-validation), MDL-S (which is the BIC), MDL-opt (the optimised MDL), MDL, the Lasso (also with $\alpha$ obtained with cross-validation), and finally, the desparsified lasso \citep[Lasso-ds,][]{Geer:2014,Dezeure:2015}.

The measures we use to evaluate the nodewise selection is the proportion of edges in a neighbourhood that is too low (under fit), correct or too high (over fit) in terms of the estimated edges in a single neighbourhood. Then, under, correct and over fit, respectively, for $R$ simulations are defined as 
\begin{align*}
\text{proportion under fit} &:=  \frac{1}{R}\sum_{r=1}^R \mathbbm{1}\{ \hat{S} \subset S\},\\  
\text{proportion correct fit} &:= \frac{1}{R}\sum_{r=1}^R \mathbbm{1}\{ \hat{S} = S\}, \quad\text{and}\\
\text{proportion over fit} &:= \frac{1}{R}\sum_{r=1}^R \mathbbm{1}\{ \hat{S} \supset S\},
\end{align*}
where $\hat{S}$ is the set of indices for the estimated non-zero edges (support) and $S$ is the true support. Given our objective of selecting a neighbourhood with high probability which includes correct edges but no spurious connections (see Section \ref{sec:nodewise-graphical-model}), we want the under and correct fit to be as large as possible and the over fit as small as possible. To look into more detail of over fitting, we also consider the false positive rate, defined as
\begin{align*}
\text{fasle positive rate} &:= \frac{1}{R}\sum_{r=1}^R \frac{| \hat{S} \cap S^c |}{ | S^c|},
\end{align*}
where $S^c$ is the complement of the support, the true negatives or zero edges. This measure gives an indication of the rate of selecting too many edges in a neighbourhood in the graph (over fitting).

\subsection{Results}\noindent
We start with the impact of possible neighbourhood size ($p/n = 0.1$ up to 10) on MSE. Figure \ref{fig:mse-correct-missp}(a) shows the MSE for a correctly specified linear model, and for a misspecified linear model in Figure \ref{fig:mse-correct-missp}(b). The test MSE in (a) remains above the optimal MSE at the correct neighbourhood size with $p/n\approx 0.15$. This implies that standard model selection with MDL-S or the AIC should work in the high-dimensional setting. However, in Figure \ref{fig:mse-correct-missp}(b), when the model is misspecified, the test MSE drops below the MSE at the optimal solution, indicating that standard neighbourhood selection need not work. Since the MSE is lower at models with a large number of parameters, neighbourhood selection could favour models with a large number of parameters. 
\begin{figure}
\begin{tabular}{@{\hspace{-1em}} c @{\hspace{1em}} c @{\hspace{0em}} c}
\vspace{-3em}
\textsf{}		&\textsf{correct model}			&\textsf{misspecified model}\\[0em]
\raisebox{8em}{\textsf{}}		&\pgfimage[width=.5\textwidth]{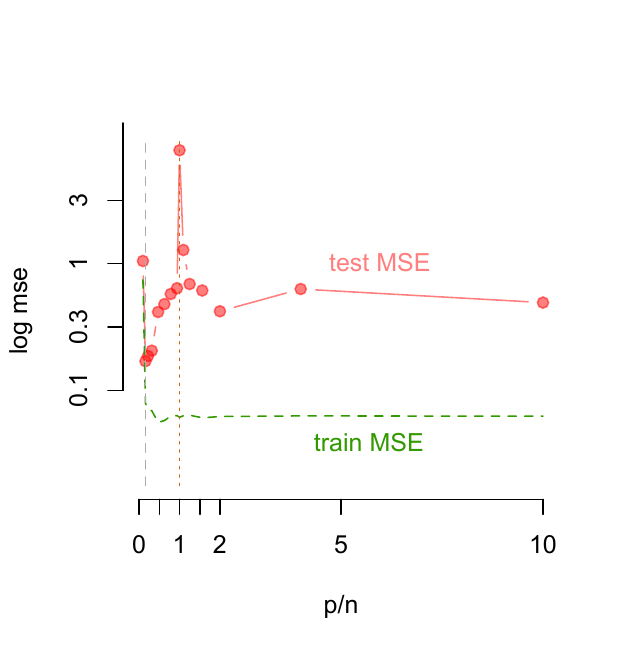}	&\pgfimage[width=.5\textwidth]{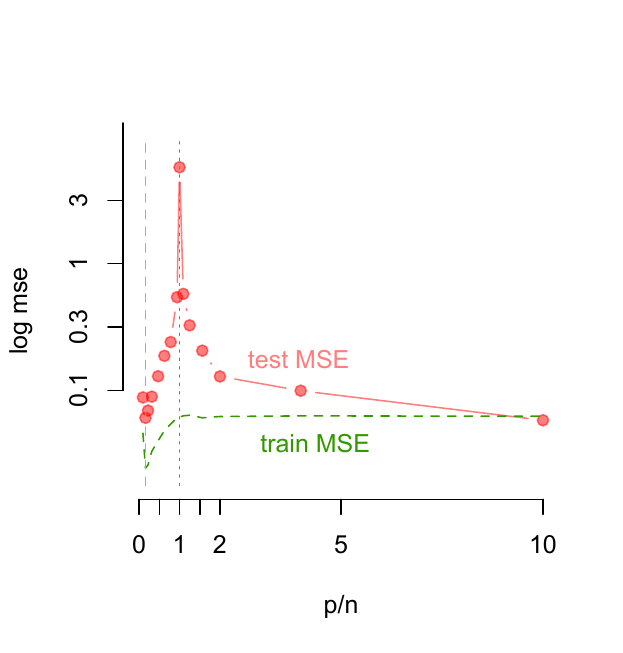}\\[-1em]
	&(a)		&(b)
\end{tabular}
\caption{The MSE for correctly specified (a) and misspecified model (b) for a neighbourhood with $p$ edges up to 320, so that $p/n$ ranges between 0.1 and 10. For the true linear model (a) the global minimum of the test MSE remains the correct model with $p=5$, approximately at $p/n\approx 0.15$. For the misspecified model (b), however, we see that the global minimum has now shifted from the correct neighbourhood size with $p=5$ and MSE 0.0877 to the model with $p/n=10$ ($=320/32$) and MSE 0.0844. }
\label{fig:mse-correct-missp}
\end{figure}
%


The consequences for neighbourhood selection are shown in Figure \ref{fig:proportion-correct-missp}. The figure shows under fit ($\hat{S}\subset S$, light blue), correct fit ($\hat{S}=S$, blue) and over fit ($\hat{S} \supset S$, red) in each subfigure. The left column of Figure \ref{fig:proportion-correct-missp} shows the results for when the data is correctly specified as linear, while the right column shows the results for when the model is misspecified; both are estimated with a linear model. In the bottom row, neighbourhood selection is performed when $p_{\max} = 15$, and increases to $p_{\max}=32$ in the middle and $320$ in the top row. When the model is correctly specified (left column) and the dimension is low ($p_{\max}=15$), neighbourhood selection in graphical models works reasonably well in that either the correct support is obtained, $\hat{S}=S$ or it is underestimated, i.e., $\hat{S}\subset S$. Only the AIC and the Lasso show increased false positive rates, i.e., $\hat{S}\supset S$. Increasing the dimension to $p_{\max}=320$ exacerbates the over fit for AIC and Lasso (also Lasso-ds). In contrast, MDL, MDL-opt, MDL-S and AIC-CV appear to remain accurate in the sense that over fit remains low, and over fit is absent in the case of MDL. For MDL this is because of the large penalty on using a large dimensional model space described in Section \ref{sec:minimum-description-length}, in line with Proposition \ref{prop:mdl-underfit}.
\begin{figure}
\begin{tabular}{@{\hspace{-1em}} c @{\hspace{1.5em}} c @{\hspace{0em}} c}
\textsf{}		&\textsf{correct model}			&\textsf{misspecified model}\\[0em]
\raisebox{8em}{\textsf{}}		&\pgfimage[width=.5\textwidth]{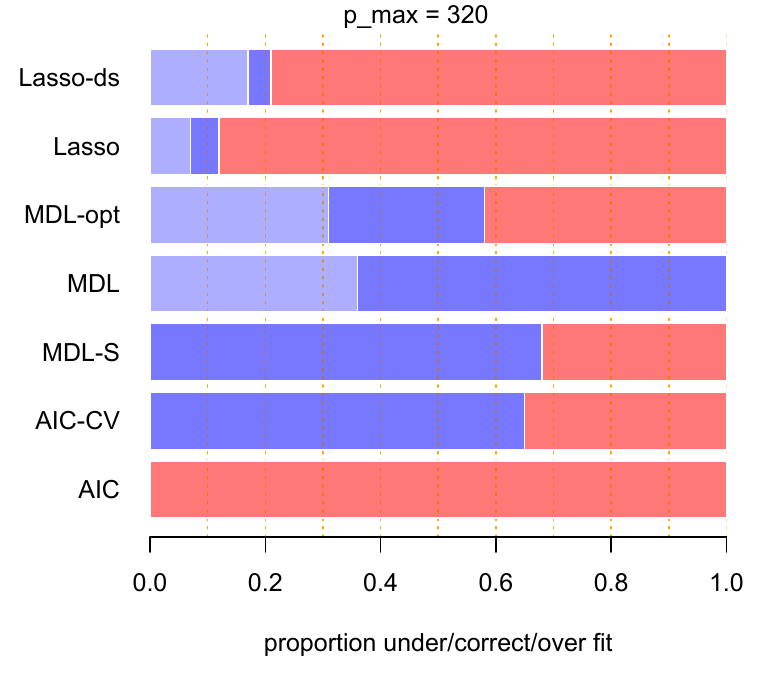}		&\pgfimage[width=.47\textwidth]{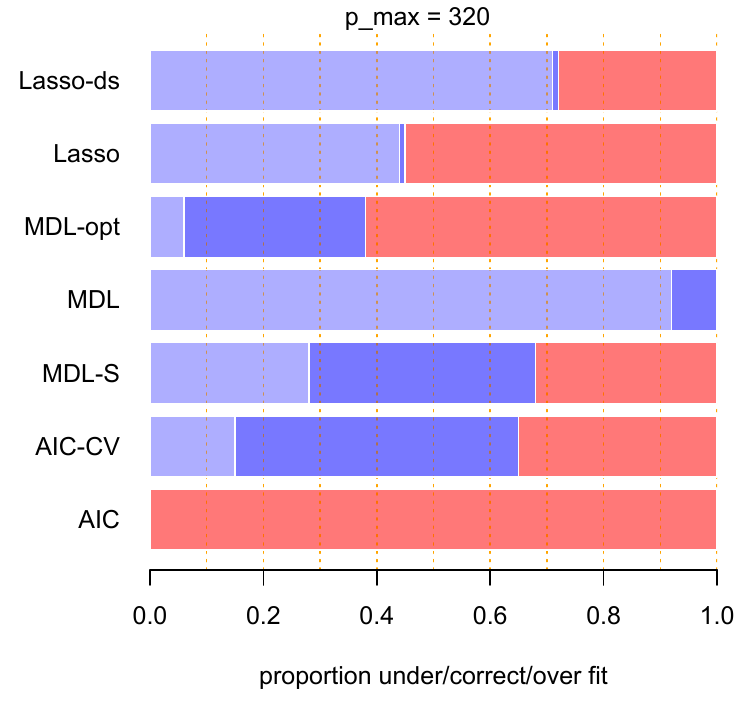}\\[-1em]
\raisebox{8em}{\textsf{}}		&\pgfimage[width=.5\textwidth]{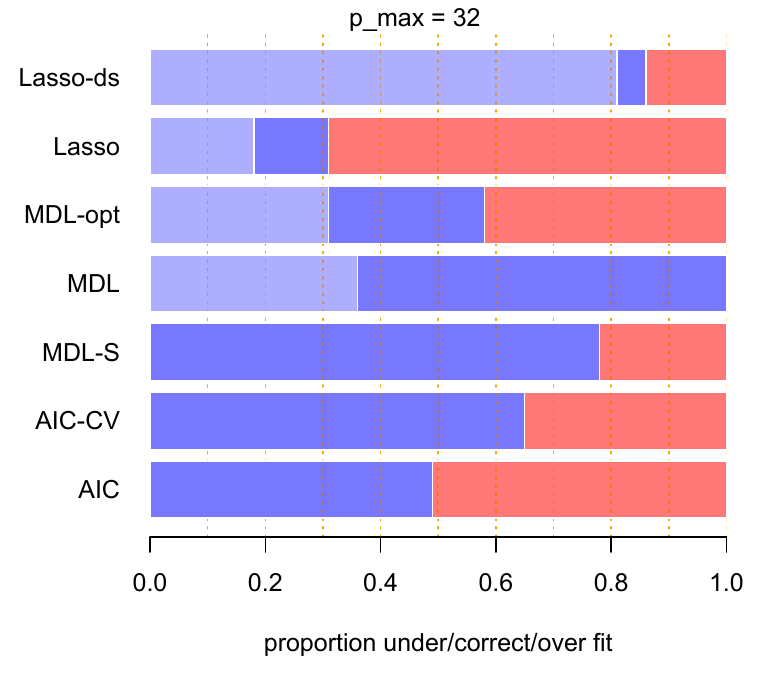}		&\pgfimage[width=.47\textwidth]{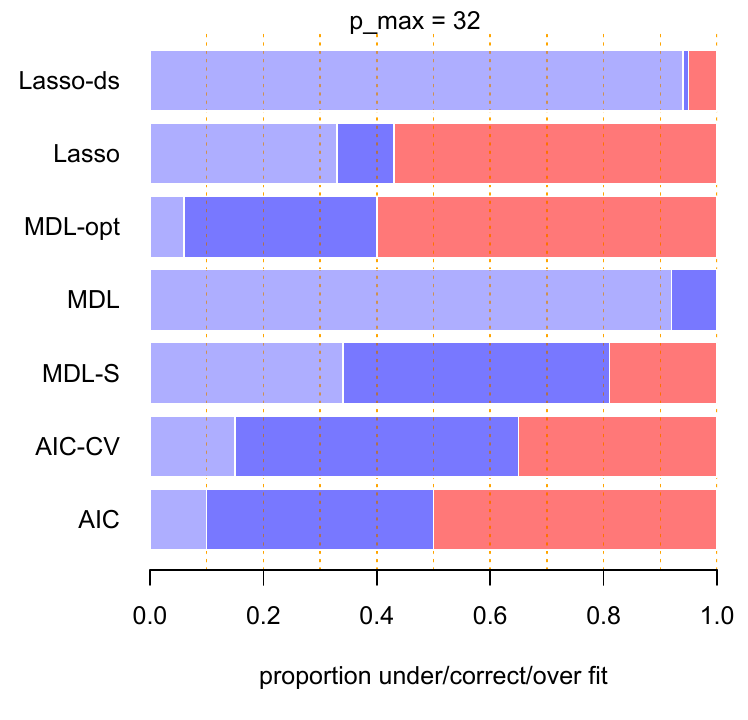}\\[-1em]
\raisebox{8em}{\textsf{}}		&\pgfimage[width=.5\textwidth]{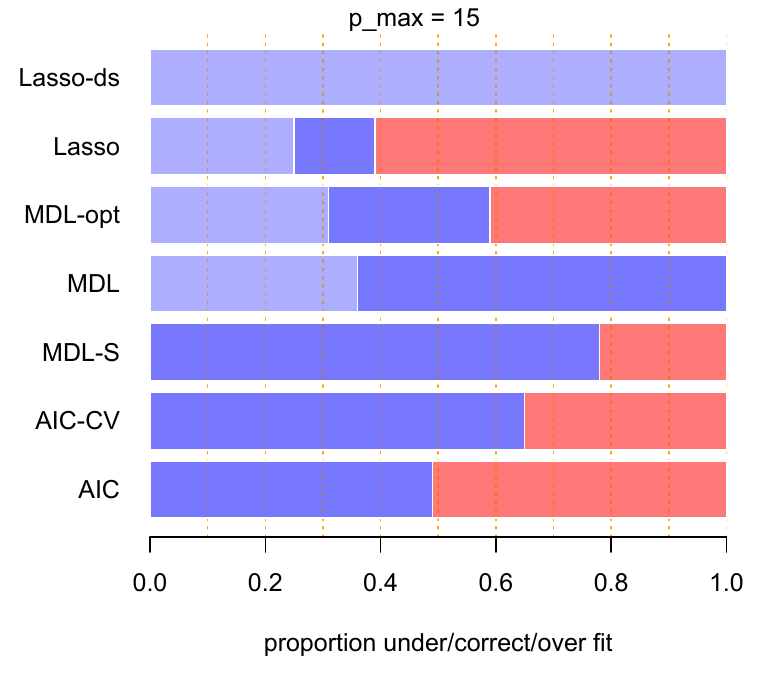}		&\pgfimage[width=.47\textwidth]{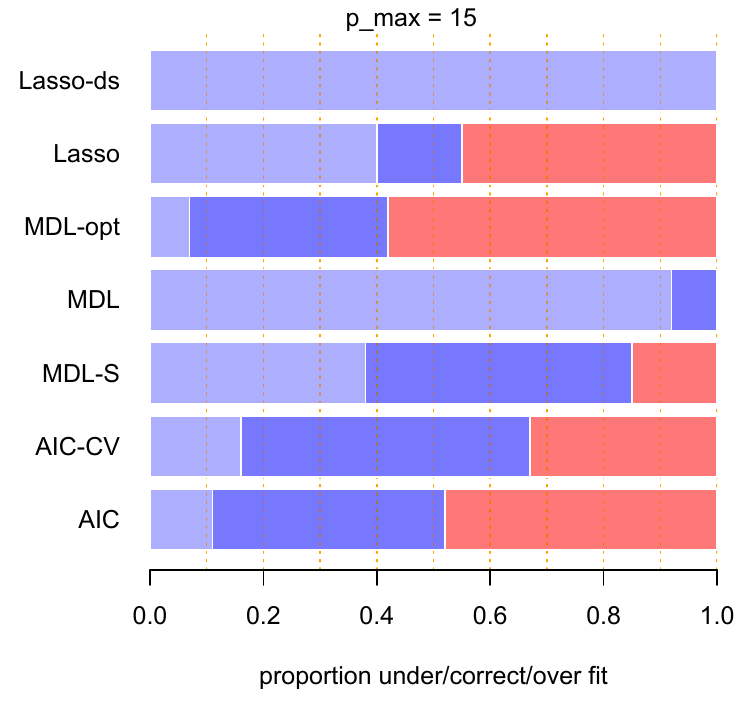}\\
\end{tabular}
\caption{Proportion of decisions: underestimating (light blue), correct (blue), and overestimating (red) the dimension of the model for both the correct (linear) and incorrect (logistic) data generating model for different maximum number of predictors $p_{\max} =15, 32$ or $320$ in each of the respective rows, such that $p/n=0.47$, 1 and 10, respectively. In both the linear and non-linear model, there are $p=5$ non-zero coefficients. In the left column the data were generated and estimated with the linear model; in the right column, data were generated with the non-linear model while they were estimated with the linear model.  }
\label{fig:proportion-correct-missp}
\end{figure}

The pattern of over fit remains relatively similar when the model is misspecified (right column of Figure \ref{fig:proportion-correct-missp}), but the probability of selecting the correct neighbourhood is decreased for all methods. Only MDL does not select too many edges (no false positives), which is in line with theory as discussed in Section \ref{sec:minimum-description-length} and Appendix \ref{sec:appendix-mdl-underfit}.

Finally, we consider the false positive rate, the probability of including too many edges in the graph. Figure \ref{fig:false-positive-rate}(a) shows the false positive rate for each dimension in the simulation $p/n=0.1$ up to $10$, when the model is correctly specified as linear.
\begin{figure}
\begin{tabular}{@{\hspace{-1em}} c @{\hspace{1em}} c @{\hspace{0em}} c}
\textsf{}		&\textsf{correct model}			&\textsf{misspecified model}\\[0em]
\raisebox{8em}{\textsf{}}		&\pgfimage[width=.5\textwidth]{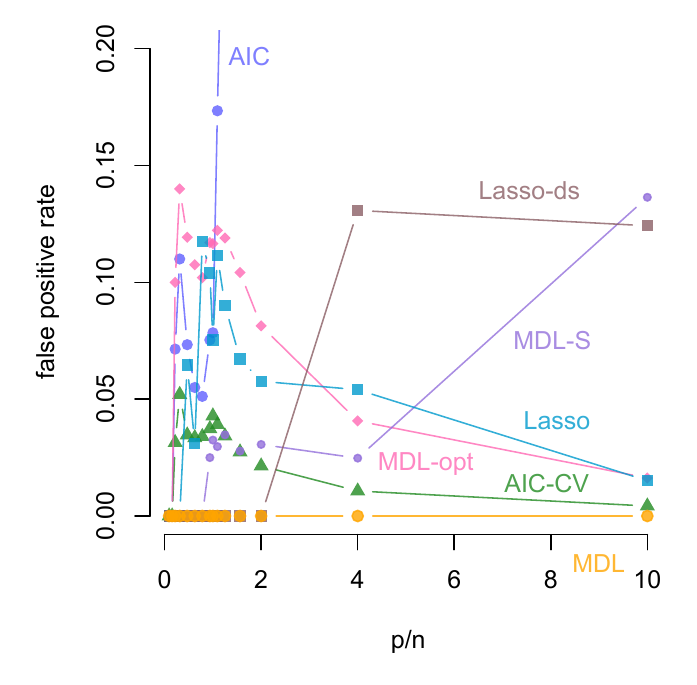}		&\pgfimage[width=.5\textwidth]{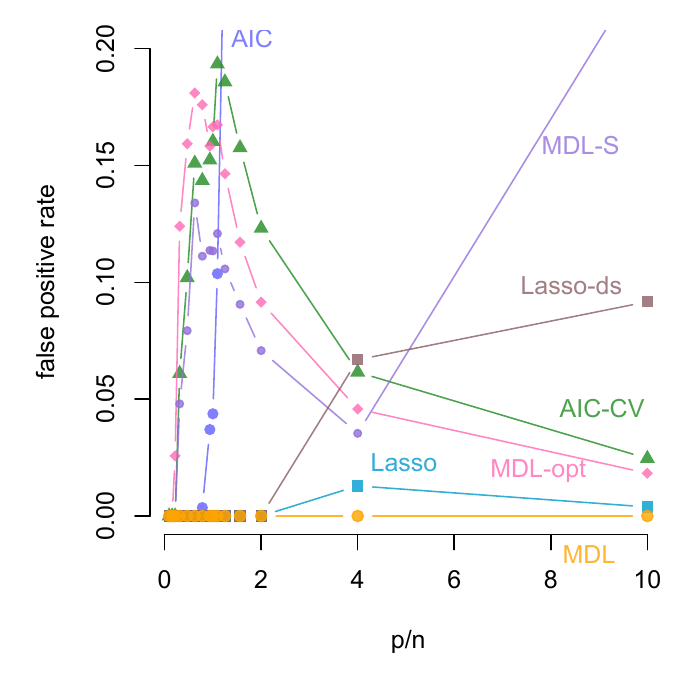}\\[-1em]
	&(a)		&(b)
\end{tabular}
\caption{False positive rate for each method. In (a) the model is correctly specified as linear and in in (b) the model is misspecified as linear while the conditional mean is a logistic function. }
\label{fig:false-positive-rate}
\end{figure}
MDL stands out in that no matter the model dimension, the false positive rate remains at 0. This corresponds to the results in the left column of Figure \ref{fig:proportion-correct-missp}. Most other methods have a slightly elevated false positive rate; the AIC tends to lead to a neighbourhood that is too large from around the interpolation point ($p/n=1$) on. When the assumption of a true linear model no longer holds, the false positive rate increases for almost all methods, except Lasso and MDL. Both MDL and Lasso have low false positive rate for any ratio of $p/n$.

\section{Conclusion and discussion}\label{sec:conclusions}\noindent
We considered nodewise selection of edges in the Gaussian graphical model when the number of nodes exceeds the number of observations. This is an issue because regular estimation techniques like least squares cannot work in such situations. We used the ridge regression estimate to solve this issue, leading to a biased but closed form estimate. We were then able to show the impact of the ridge parameter on the mean squared error, and concluded that indeed, when the true relation between edges is linear, the MSE is lowest at the correct neighbourhood. We also showed that in the misspecified case, the MSE is co-determined by an additional term in the MSE caused by the misspefication of a nonlinear model by a linear model. In addition, the ridge parameter could be interpreted as the inverse of the signal-to-noise-ratio. In low signal-to-noise-ratio situations the ridge parameter will be high, leading to a larger ridge penalty. This penalty was then seen to reduce the parameter variance and, hence, implies that the generalisation can still be reasonable, suggesting generalisation to other samples is possible. This also (partly) explains why some highly overparameterised machine learning techniques are able to predict well, mimicking regularised estimation by restricting the class of functions in empirical risk minimisation \citep{Hastie:2022}. Generally, selecting the ridge parameter using cross-validation seems good practice, since it reduced the sensitivity of the MSE and neighbourhood selection improved.

We next applied nodewise selection to ascertain whether it is possible to obtain the correct neighbourhood or smaller, but not larger (i.e., no false positives). We proved that in the high-dimensional setting MDL can achieve this objective of correct or smaller neighbourhood selection with high probability, in both the case when true edge relations are linear or nonlinear. The reason is that it explicitly takes into account the volume of the parameter space to counteract the low MSE at high dimensions. Simulations confirmed these results, and also showed that other neighbourhood selection methods (using AIC, for example) lead to neighbourhoods that are too large. Obviously, a low false positive rate and under fit for MDL, lead to a somewhat elevated false negative rate (too few correct edges). Considering the discussion in Section \ref{sec:nodewise-graphical-model}, we believe that, even though some edges will not be detected, this is preferable to the case when too many edges are selected. Therefore, we recommend using MDL for nodewise selection in the high-dimensional setting for Gaussian graphical modelling. 

While the result that MDL leads to low false positives is geared towards the application of the high-dimensional setting in graphical modelling, it is more general. We used the framework of nodewise regression, and so it is valid within the context of regression. But we used the fact that there is a bound on the misspecification for the conditional mean and that the residual variance is bounded away from 0. Given these settings, the result on MDL can be applied to other regression settings. 

\section*{Acknowledgements}\noindent
The research conducted by L. Waldorp has been supported by a grant from the
ERC (ERC project 101053880) and by the “New Science of Mental Disorders”, the Dutch
Research Council and the Dutch Ministry of Education, Culture and Science (NWO), grant number 024.004.016.


\appendix 
\vspace{2em}\noindent
\textbf{Appendix}
%
\appendix 
\vspace{2em}\noindent

\section{Gaussian graphical model}\label{sec:appendix-gaussian-graphical-model}\noindent
The Gaussian graphical model (GGM) is characterised by the correspondence between a network of $p$ nodes, labeled $j=1,2,\ldots, p$ and a set of random variables $X_1,X_2,\ldots, X_p$ that are jointly Gaussian (multivariate normally) distributed \citep{Koller:2009}. The edges in a GGM correspond to conditional dependence between nodes (given all reamining variables). The Gaussian distribution is completely described by its means $\mu_j$ and covariances $\sigma_{ij}$. The inverse of the covariance matrix $\Sigma=(\sigma_{ij},i,j=1,2,\ldots,p)$, denoted by $\Sigma^{-1}=\Theta$, contain the partial covariances. The partial covariance $\theta_{ij}$ is defined as the covariance between the residuals of nodes $i$ and $j$ when all other variables (not $i$ and $j$) have been regressed out \citep{Lauritzen96}. In a GGM a partial covariance of 0 is equal to conditional independence. We show this by example. 

Consider three Gaussian variables with mean 0 and covariance matrix $\Sigma$. We write the density of the three variables as $f_{123}$ and the conditional density as $f_{13\mid 2}$ for the joint density of variables $X_1$ and $X_3$ given $X_2$. We will want to come to the conclusion that the product of the conditional distributions $f_{1\mid 2}f_{3\mid 2}$ is the same as the conditional distribution $f_{13\mid 2}$.  To do this we consider an example with covariance and inverse covariance matrix, respectively, 
\begin{align*}
\Sigma = 
\begin{pmatrix}
3/4	&-1/2	&1/4\\[-.5em]
-1/2	&1		&-1/2\\[-.5em]
1/4	&-1/2	&3/4
\end{pmatrix}
\qquad\text{and}\qquad
\Theta
 = 
\begin{pmatrix}
2\phantom{m}	&1\phantom{m}	&0\\[-.5em]
1\phantom{m}	&2\phantom{m}	&1\\[-.5em]
0\phantom{m}	&1\phantom{m}	&2
\end{pmatrix}.
\end{align*}
Note the 0 at $\theta_{13}=\theta_{31}$ such that variables $X_1$ and $X_3$ are conditionally independent given $X_2$. In terms of the distribution this implies the following. So this implies that we have the density $f_{13\mid 2}$. We now compute this from the density of thew corresponding conditional Gaussian distribution. The distribution with $\Theta$ is
\begin{align*}
f_{123}(x) = \frac{1}{\sqrt{(2\pi)^3}}\sqrt{\det(\Theta)} \exp\left( -\tfrac{1}{2} x^\top \Theta x \right),
\end{align*}
where the term in the exponential is
\begin{align*}
x^\top \Theta x = 2x_1^2 + 2x_2^2 + 2x_3^2 + 2 x_1x_2 + 2x_2x_3.
\end{align*}
Note that because $\theta_{13}=0$ there is no term $2x_1x_3$ and that is why we can rewrite the density. If we take the density of $X_1$ and $X_2$, using the first two rows and columns of $\Theta$ and calling that $\Theta_{12}$, then we obtain
\begin{align*}
f_{12}(x_1,x_2) = \frac{1}{2\pi} \det(\Theta_{12})\exp\left( -\tfrac{1}{2}(2x_1^2 + 2x_1x_2 +2x_2^2)\right).
\end{align*}
And similarly for the density of $X_2$ and $X_3$ with $\Theta_{23}$ the second and third rows and columns of $\Theta$, gives
\begin{align*}
f_{23}(x_2,x_3) = \frac{1}{2\pi} \det(\Theta_{23})\exp\left( -\tfrac{1}{2}(2x_2^2 + 2x_2x_3 +2x_3^2)\right).
\end{align*}
The density of only $X_2$, ignoring the other two variables is
\begin{align*}
f_{2}(x_2) = \frac{1}{\sqrt{2\pi}} \Theta_{22}\exp\left( -\tfrac{1}{2}2x_2^2\right).
\end{align*}
Then we find (after some algebra) that
\begin{align*}
\frac{f_{12}f_{23}}{f_{2}}=f_{123}
\quad \Longleftrightarrow \quad
\frac{f_{12}f_{23}}{f_{2}f_2}=f_{1\mid 2}f_{3\mid 2}=f_{13\mid 2}.
\end{align*}
And we see that if a conditional covariance $\theta_{ij}=0$, then there is a conditional independence. This is true only for the multivariate normal distribution.

Hence, the GGM is particularly attractive for graphical models because whenever the partial correlation between variables $X_i$ and $X_j$ is 0, then $X_i$ and $X_j$ are conditionally independent given all other variables \citep[][Proposition 5.2]{Lauritzen96}. Hence, an edge in the GGM is present if and only if the partial correlation is non-zero. 
The partial correlation is a function of the inverse covariance matrix and has elements $\theta_{ij}$. The partial correlation between $X_i$ and $X_j$ is defined by \citep{Koller:2009,Epskamp:2018b}
\begin{align*}
\rho_{ij\mid \star} = -\frac{\theta_{ij}}{\sqrt{\theta_{ii}\theta_{jj}}},
\end{align*}
where $\star$ indicates all remaining variables in the network except $i$ and $j$. If $\rho_{ij\mid \star}=0$, then no edge should be present between nodes $i$ and $j$ in the network. A GGM can be interpreted as a network where a correlation corresponds to a (set of) path(s) between two nodes \citep{Jones:2005}.

We can determine whether a partial correlation also by considering the regression coefficients $\beta_{ij}$ from the regression 
\begin{align*}
X_i = \sum_{j\ne i} X_{j}\beta_{ij} + e_i,
\end{align*}
where the sum is over all nodes $j$ that are not $i$. The reason for the correspondence between the partial correlation and the regression coefficient is that \citep{Lauritzen96}
\begin{align*}
\beta_{ij} = -\frac{\theta_{ij}}{\theta_{ii}}.
\end{align*}
Hence, whenever $\beta_{ij}=0$ then also $\theta_{ij}=0$ and vice versa. We can, therefore, consider each node in turn and determine the non-zero coefficients of the regressions. Because we have both $\beta_{ij}$ and $\beta_{ji}$ we will have to decide to use the \textit{and}-rule, where both $\beta_{ij}$ and $\beta_{ji}$ are non-zero to include the edge $(i,j)$, or the \textit{or}-rule, where either $\beta_{ij}$ or $\beta_{ji}$ is non-zero to include the edge $(i,j)$ \citep{Meinshausen:2006}.

\section{Excess risk and mean squared error (Section \ref{sec:mse-decomposition})}\label{sec:appendix-mean-squared-error}\noindent
We defined the excess risk as
\begin{align*}
R(J)	
&=\mathbb{E}\left[ (Y_0-X_0^\top\hat{\beta}_{J} )^2 - (Y_0 - X_0^\top \beta_S)^2\right]
\\
&=\mathbb{E}\left[ (Y_0-X_0^\top\hat{\beta}_{J} )^2\right] - \mathbb{E}\left[(Y_0 - X_0^\top \beta_S)^2\right].
\end{align*}
where the expectation is conditional on the training data $X$. This is equivalent to the risk defined in \citet{Hastie:2022} and \citet{Bartlett:2021}.
Then we obtain the following decomposition.
%
\begin{lemma}{\rm (Mean squared error decomposition)}\label{lem:mse-decomposition}
Assume that the linear model (\ref{eq:linear-model}) is correct. The excess risk (test mean squared error) can be decomposed as $R(J)=B(J)+V(J)$ in (\ref{eq:risk}) with
\begin{align*}
B(J) = (\beta_S -\mathbb{E}\hat{\beta}_J)^\top \Sigma (\beta_S -\mathbb{E}\hat{\beta}_J), 
\qquad
V(J) =  \mathbb{E}(\hat{\beta}_J-\mathbb{E}\hat{\beta}_J)^\top \Sigma (\hat{\beta}_J-\mathbb{E}\hat{\beta}_J),
\end{align*}
where $B(J)$ is called the squared bias and $V(J)$ is called the variance.
\end{lemma}
\begin{proof}
Assuming that the linear model is correct, so that $\mathbb{E}Y_0-X_0^\top\beta_S=0$, we get
\begin{align*}
R(J) 
&=\mathbb{E}\left[ (Y_0-X_0^\top\beta_S+X_0^\top(\beta_S-\hat{\beta}_{J}) )^2\right] - \mathbb{E}\left[(Y_0 - X_0^\top \beta_S)^2\right]
\\
& = \mathbb{E}\left( X_0^\top(\beta_S -\mathbb{E}\hat{\beta}_J) - X_0^\top(\hat{\beta}_J-\mathbb{E}\hat{\beta}_J)\right)^2 
\end{align*}
This gives 
\begin{align*}
R(J) 
& = \mathbb{E}\left( X_0^\top(\beta_S -\mathbb{E}\hat{\beta}_J)\right)^2 + \mathbb{E}\left(X_0^\top(\hat{\beta}_J-\mathbb{E}\hat{\beta}_J)\right)^2 \\
&\quad
- 2\mathbb{E} (\beta_S -\mathbb{E}\hat{\beta}_J)^\top X_0 X_0^\top (\hat{\beta}_J-\mathbb{E}\hat{\beta}_J)
\end{align*}
Recalling that $(X,Y)$ and $(X_0,Y_0)$ are independent, we obtain for the cross product
\begin{align*}
&\mathbb{E} (\beta_S -\mathbb{E}\hat{\beta}_J)^\top X_0 X_0^\top (\hat{\beta}_J-\mathbb{E}\hat{\beta}_J)
= (\beta_S -\mathbb{E}\hat{\beta}_J)^\top \Sigma \mathbb{E} (\hat{\beta}_J-\mathbb{E}\hat{\beta}_J) = 0,
\end{align*}
because $\mathbb{E} (\hat{\beta}_J-\mathbb{E}\hat{\beta}_J) = 0$, and where $\Sigma  =\mathbb{E}(X_0 X_0^\top)$, which is the same as the covariance matrix of the predictors $\mathbb{E}(X^\top X)$. Then we obtain $B(J)$ and $V(J)$ in (\ref{eq:bias-variance}) from the two remaining terms.
\end{proof}
This is the famous variance and squared bias decomposition \citep{Hastie:2022}.

Next, we show the formulation of the bias and variance of the MSE with the ridge estimator $\hat{\beta}=W(\alpha)X^\top Y$, where $W(\alpha)=(X^\top X+\alpha I)$. 
\begin{lemma}\label{lem:bias-eigen-values}
Assume that the linear model (\ref{eq:linear-model}) is true and that $\Sigma=\sigma_\xi^2 I$ is the covariance matrix of the covariates. Then  
 the squared bias is 
\begin{align*}
B(J) = \sigma^2_\xi\frac{\gamma_1\alpha^2}{(\lambda_1+\alpha)^2}+ \sigma^2_\xi\sum_{j=2}^n\frac{\alpha^2}{(\lambda_j+\alpha)^2}.
\end{align*}
 \end{lemma}
\begin{proof}
Recall that the expectation is conditional on $X$ and $(X_0,Y_0)$ and $(X,Y)$ are independent. Plugging in the ridge estimator (\ref{eq:ridge-estimator}) in the bias term and assuming $\Sigma=\sigma_\xi I$, yields
\begin{align*}
B(J) &= \mathbb{E}(\beta_S -\mathbb{E}\hat{\beta}_J)^\top X_0^\top X_0 (\beta_S -\mathbb{E}\hat{\beta}_J) \\
&=
\text{tr}( \beta_S - W(\alpha)^{-1}X^\top X\beta_S)( \beta_S - W(\alpha)^{-1}X^\top X\beta_S)^\top \mathbb{E}(X_0^\top X_0)
\\
&=
\text{tr}( (I - W(\alpha)^{-1}X^\top X)\beta_S)( (I - W(\alpha)^{-1}X^\top X)\beta_S)^\top \Sigma
\\
&=
\sigma^2_\xi \text{tr}( I - W(\alpha)^{-1}X^\top X)^2\beta_S\beta_S^\top.
\end{align*}
Let $\lambda_j$ be the eigenvalues of $X^\top X$ and $\gamma_1$ be the only non-zero eigenvalue of $\beta_S\beta_S^\top$. Then, noting that $U^\top U=I$, we obtain for $W(\alpha)=X^\top X+\alpha I$ the eigenvalue decomposition for $W(\alpha)$ is $U\Lambda(\alpha) U^\top$ with diagonal matrix $\Lambda(\alpha)$ with elements $\lambda_j + \alpha $. The squared  inverse $W(\alpha)^{-2}$ is then $U\Lambda(\alpha)^{-2} U^\top$ with diagonal matrix $\Lambda(\alpha)^{-2}$ with elements $1/(\lambda_j+\alpha)^2$. Plugging this in the above result, we obtain $B(J)$. This completes the proof.
\end{proof}
%
%
\begin{lemma}\label{lem:variance-eigen-values}
Assume that the linear model (\ref{eq:linear-model}) is true and that $\Sigma=\sigma_\xi^2 I$ is the covariance matrix of the covariates. Then  
the variance is 
\begin{align*}
V(J) = \sigma_\xi^2\sigma^2_e \sum_{j=1}^n \frac{\lambda_j} {(\lambda_j + \alpha)^2}.
\end{align*}

\end{lemma}
\begin{proof}
Again, we note that the expectation is conditional on $X$ and $(X_0,Y_0)$ and $(X,Y)$ are independent. Plugging in the ridge estimator (\ref{eq:ridge-estimator}) in the variance term and assuming $\Sigma=\sigma_\xi I$, yields
\begin{align*}
V(J) &= \mathbb{E}\text{tr}(\hat{\beta}_J-\mathbb{E}\hat{\beta}_J)(\hat{\beta}_J-\mathbb{E}\hat{\beta}_J)^\top X_0^\top X_0
\\
&=\text{tr}(W(\alpha)^{-1}X^\top Y-W(\alpha)^{-1}X^\top X\beta_S)\\
	&\qquad \times (W(\alpha)^{-1}X^\top Y-W(\alpha)^{-1}X^\top X\beta_S)^\top \mathbb{E}(X_0^\top X_0)
\\
&=\text{tr}W(\alpha)^{-2}(X^\top (X\beta_S + e- X\beta_S))(X^\top (X\beta_S + e- X\beta_S))^\top \Sigma
\\
&=\sigma_\xi^2 \sigma^2_e\text{tr}(W(\alpha)^{-2} X^\top X),
\end{align*}
where we used that the variance of the residuals $e$ is $\sigma^2_e I$. Similar to Lemma \ref{lem:bias-eigen-values} we use that the eigenvalue decomposition for $W(\alpha)$ is $U\Lambda(\alpha) U^\top$ with diagonal matrix $\Lambda(\alpha)$ with elements $\lambda_j + \alpha $. The squared  inverse $W(\alpha)^{-2}$ is then $U\Lambda(\alpha)^{-2} U^\top$ with diagonal matrix $\Lambda(\alpha)^{-2}$ with elements $1/(\lambda_j+\alpha)^2$. Plugging this in the above results, we obtain $V(J)$. This completes the proof.
\end{proof}

\subsection{Model misspecification}\label{sec:appendix-misspecification}\noindent
Misspecification is defined with respect to a family of models. Let $\mathcal{M}$ denote a class of probability distributions that are indexed by parameters $\theta\in \Theta\subseteq \mathbb{R}^p$, i.e., $\mathcal{M}=\{ p_\theta:\theta\in \Theta \}$ \citep[we assume that the probability distribution is dominated by Lebesgue measure so that a density $p_\theta$ exists,][]{Jacod:2004}. An example of a model is  a class of normal distributions with mean $\mu$ and variance $\sigma^2$, where the mean is generated by a nonlinear function, such as $f:\mathbb{R}^p\times \mathbb{R}^p\to \mathbb{R}$ such that $x\mapsto 1/(1+\exp(-x^\top \beta))$; then $\theta=(\beta,\sigma^2)$ is the $p+1$ dimensional parameter vector. Furthermore, let $\mathcal{M}_{\sf lin}\subset \mathcal{M}$ be a strict subset consisting of only those distributions where the conditional mean $\mathbb{E}(Y\mid x)$ is generated by a linear function $x\mapsto x^\top\beta$. We say that a model is misspecified if the correct distribution is not in the class of probability distributions we consider in the optimisation  \citep{Vandervaart98,Grunwald:2007,Kleijn:2026}, i.e., $p_0\notin \mathcal{M}_{\sf lin}$, where $p_0=p_{\theta_0}$.  In our setting, this means that the distribution has a mean that is generated by a nonlinear function, while we use only a linear function for optimisation. In such a case, the linear model $x^\top \hat{\beta}$ can be interpreted as corresponding to a density $p_\theta$ that is closest within $\mathcal{M}_{\sf lin}$ to $p_0$ in $\mathcal{M}$ in terms of Kullback-Leibler divergence \citep{Vuong:1989,Vandervaart98}. Here we have suppressed the variance parameter $\sigma^2$ of the residuals $Y-f$, where $f$ is the true linear model for the conditional mean. However, seminal work by \citet{Grunwald:2017} shows that misspecification using linear models, as we do here, also leads to misspecification of the variance $\sigma^2$ (i.e., it depends on the value of $x$), and leads to non-convexity of the model space $\mathcal{M}_{\sf lin}$. This in turn leads to models that can be combinations of linear models (like with variance depending on $x$) that are in the convex hull of the model space $\mathcal{M}_{\sf lin}$ and are closer to the true generating model not in $\mathcal{M}_{\sf lin}$ than any model in $\mathcal{M}_{\sf lin}$ itself; \citet{Grunwald:2017} called this bad misspecification. 

We approximate a function $f$ by a linear function $X\beta$ with $\beta\in \mathbb{R}^p$ and $p>n$ very large. 
\begin{lemma}{\rm (Mean squared error for misspecification)}\label{lem:mse-misspecification}
Assume that the true function $f:\mathbb{R}^p\times \mathbb{R}^p\to \mathbb{R}$, and $Y=f(X,\beta_S)+e$, with $\beta_S$ the true parameter vector in $f$. We estimate $f(X,\beta_S)$ by $X\beta_J$ with $|J|=p>n$. Then the risk under model misspecification is 
\begin{align*}
R_f(J) &= \mathbb{E}\left[ (f(X_0,\beta_S) - \mathbb{E}(X_0^\top\hat{\beta}_J)^2 \right]
	+ \mathbb{E}\left[ (f(X_0,\beta_S) - \mathbb{E}(X_0^\top\hat{\beta}_J)^2 \right]\\
	&\quad
	+ 2\mathbb{E}\left[ (Y_0-f(X_0,\beta_S))(f(X_0,\beta_S)-X_0^\top\hat{\beta}_J) \right]
\end{align*}
which we refer to as $R_f(J)=B_f(J)+V_f(J)+M_f(J)$.
\end{lemma}
\begin{proof}
The proof follows Lemma \ref{lem:mse-decomposition} except that the cross products do not vanish because of misspecification. We have that the excess risk is
\begin{align*}
R_f(J) 
&=\mathbb{E}\left[ (Y_0-f(X_0,\beta_S)+f(X_0,\beta_S)-X_0^\top\hat{\beta}_{J}) )^2\right] - \mathbb{E}\left[(Y_0 - f(X_0, \beta_S))^2\right].
\end{align*}
Working out the cross products cancels $\mathbb{E}\left[(Y_0 - f(X_0, \beta_S)^2\right]$, without assuming that the model is correct. What remains is $B_f(J)$, $V_f(J)$, similarly obtained as in Lemma \ref{lem:mse-decomposition}, but the cross product that remains is
\begin{align*}
M_f(J) =
2\mathbb{E}\left[ (Y_0-f(X_0,\beta_S))(f(X_0,\beta_S)-X_0^\top\hat{\beta}_J) \right].
\end{align*}
This was to be proved. 
\end{proof}
\begin{proposition}{\rm (Misspecification bound)}\label{prop:misspecification-bound}
Assume the model with additive noise $Y=f(X,\beta_S)+e$, where the conditional mean $\mathbb{E}(Y\mid x)$ is given by the function is $f:\mathbb{R}^p\times \mathbb{R}^p\to \mathbb{R}$ and $e=Y-f(X,\beta_S)$ has mean zero and finite variance $\sigma^2$. Furthermore, we estimate $f(X,\beta_S)$ by $X\beta_J$ with $|J|=p>n$, and assume that the approximation error is bounded in probability by $K$, i.e., $f(X_0,\beta_S)-X_0^\top\hat{\beta}_J=O_p(K)$. Then the  misspecification is bounded by
\begin{align*}
M_f(J) &=  2\mathbb{E}\left[ \left| (Y_0-f(X_0,\beta_S))(f(X_0,\beta_S)-X_0^\top\hat{\beta}_J) \right| \right]
	\le 
	2\sigma O(K)
\end{align*}
\end{proposition}
\begin{proof}
By the Cauchy-Schwarz theorem we can bound the expectation of the product, and so, 
\begin{align*}
	\mathbb{E}\left[ | e(f(X_0,\beta_S)-X_0^\top\hat{\beta}_J) | \right]^2
	\le
	\mathbb{E} (e^2) \mathbb{E}\left[  (f(X_0,\beta_S)-X_0^\top\hat{\beta}_J)^2 \right]
\end{align*}
where we used that $Y_0-f(X_0,\beta_S)=e$. By assumption, $e$ has mean 0 and variance $\sigma^2$, and so $\mathbb{E} (e^2)=\sigma^2$. For the second term, we assumed that $f(X_0,\beta_S)-X_0^\top\hat{\beta}_J=O_p(K)$, and so 
\begin{align*}
	\mathbb{E}\left[ (f(X_0,\beta_S)-X_0^\top\hat{\beta}_J)^2 \right]=\mathbb{E}\left[ O_p(K)^2 \right] =O(K^2)
\end{align*}
Taking the square root of the product gives the required bound. 
\end{proof}
%


\section{High-dimensional space}\label{sec:appendix-high-dimensional-space}\noindent
For some intuition on high-dimensional spaces, the isotropic Gaussian, i.e., with mean 0 and variances of 1 and covariances of 0, is convenient. Let $X$ be a Gaussian random variable of dimension $p$ with mean 0 and isotropic covariance matrix $\mathbb{E}(XX^\top)=\Sigma$ of dimensions $p\times p$. We will consider an inequality that shows that with high probability, for large dimension $p$ most of the draws from an isotropic distribution are at $\sqrt{\smash[b]{p}}$, the radius of the spheroid of the distribution. 
%
%

We can make this clear by considering the probability of an observation falling within a certain bound of the origin \citep[][Chapter 3]{Vershynin:2018}. Suppose we are interested in the length (norm) of the vector $X$ in $p$ dimensions, i.e., we want to know the probability of $||X||_2=\sqrt{\smash[b]{X_1^2 + X^2_2 +\cdots X_p^2}}$ being smaller than $\sqrt{\smash[b] p}$. It is intuitive that the length of $X$ is indeed $\sqrt{\smash[b] p}$ because
\begin{align*}
\mathbb{E} ||X||_2^2 = \mathbb{E} \sum_{j=1}^p X_j^2 = \sum_{j=1}^p \mathbb{E} X_j^2=p.
\end{align*}
So, we should expect many of the values of the square $||X||_2^2$ to be around $\sqrt{\smash[b] p}$. 
Suppose we are considering the probability of $X$ being near 0, i.e., $||X||_2\approx 0$. Then we find that 
\begin{align*}
\mathbb{P}(| ||X||_2 - \sqrt{p}| \le t)\approx \mathbb{P}( \sqrt{p} \le t).
\end{align*}
So, for small values $t$, say 1 or 2, the probability can only be high if $p$ is small.

Another way to consider this is by looking at the volume of a $p$-dimensional ball. The volume of the $p$-dimensional ball in Euclidean space $\mathbb{R}^p$ with radius $r$ is \citep[][Section 1.2]{Giraud:2014}
\begin{align}\label{eq:ball-volume}
\mathbb{B}^p(r) = \frac{\pi^{p/2}}{\Gamma(p/2+1)}r^p \approx \left( \frac{2\pi e r^2}{p} \right)^{p/2} \frac{1}{p\pi}\quad \text{for large $p$},
\end{align}
where $\Gamma$ is the Gamma function $\Gamma(\alpha)=\int_0^\infty t^{\alpha-1}e^{-t}dt$ and $\alpha>0$. Figure \ref{fig:high-dimensions-volume-crust}(a) shows that the volume of the ball is close to 0 already when the dimension is about 20. Moreover, Figure \ref{fig:high-dimensions-volume-crust}(b) shows that most of the points from the Gaussian isotropic distribution will be near the surface of the $p$-dimensional ball $\mathbb{B}^p(\sqrt{\smash[b]{p}})$. The figure shows the fraction of the points outside an inner $p$-dimensional ball of radius $0.95 r$, so very close to the ball of radius $r$. The fraction increases to 1 exponentially fast \citep{Giraud:2014}. Already at $p=30$ the fraction in the last 5\% of the ball is nearly 80\%. This shows that most of the points will be near the surface of the $p$-dimensional ball.
\begin{figure}\centering
\begin{tabular}{c   c}
\pgfimage[width=.5\textwidth]{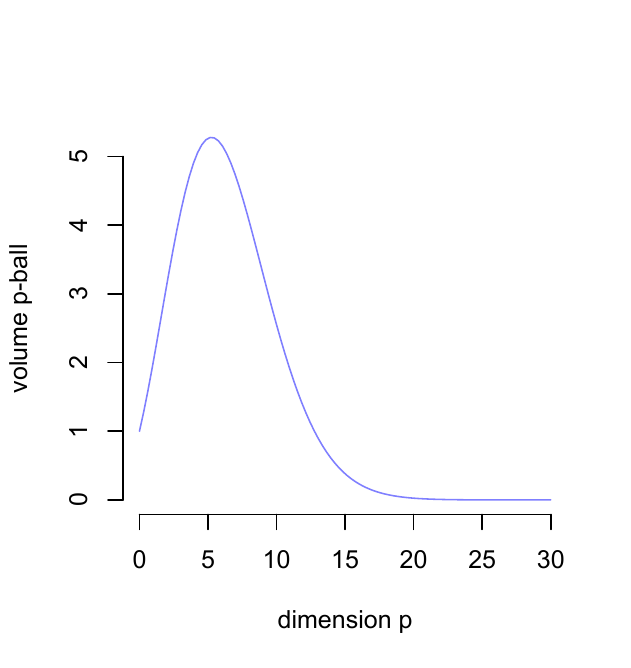}		&\pgfimage[width=.5\textwidth]{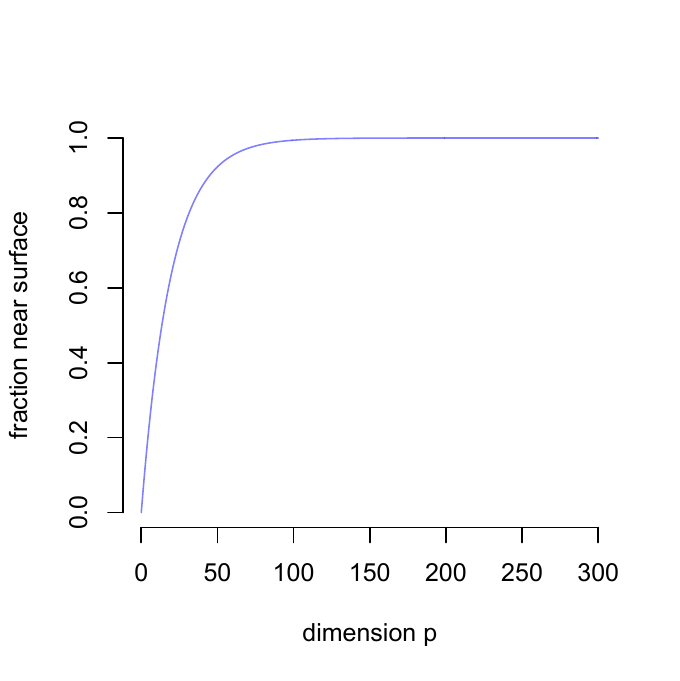}\\[-1em]
(a)	&(b)	
\end{tabular}
\caption{In (a) we see that, indeed, the volume of a $p$-dimensional ball with radius 1 in high dimensions gets smaller and smaller, already are $p=20$ is the volume close to 0. In (b) is the consequence of this, where we see that the fraction of observations from an isotropic distribution near the surface increases exponentially to 1. }
\label{fig:high-dimensions-volume-crust}
\end{figure}
%

\section{Geometry and metric entropy of Gaussian process (Section \ref{sec:model-selection})}\label{sec:appendix-geometry-metric-entropy}\noindent
Here we first prove that the the ratio of signal and noise volumes can, for linear models in the high-dimensional setting, be written in terms of the ridge estimator. This also implies that the ridge parameter $\alpha$ can be interpreted as the inverse of the signal to noise ratio.
%
%
%
\begin{lemma}{\rm (Volume linear model)}\label{lem:volume-linear-model}
We assume the linear model $X\beta$ where $X$ is an $n\times p$ matrix and $\beta$ is a $p$-dimensional vector, such that $Y=X\beta +e$ and $e$ is Gaussian with mean 0 and variance $\sigma^2_e I$. 
The volume ratio of the linear model with ridge regression with respect to the noise $e$ is 
\begin{align*}
\frac{\text{vol}(X\beta)}{\text{vol}(e)} = \sqrt{\det(X^\top X+\alpha I_p)},
\end{align*}
where $\alpha^{-1}=||\beta||_2^2/\sigma^2_e$.
\end{lemma}
\begin{proof}
Let $\text{vol}(\beta)=\mathbb{B}^n(\sqrt{n}||\beta||_2)$ be the $n$-dimensional ball with radius $\sqrt{n}||\beta||_2$. 
Compared to the volume of the noise $\mathbb{B}(\sqrt{n\sigma^2})$, we obtain the ratio
\begin{align*}
\frac{\text{vol}(X\beta)}{\text{vol}(e) }= \sqrt{\det(XX^\top)} \frac{\mathbb{B}^n(\sqrt{n}||\beta||_2)} {\mathbb{B}^n(\sqrt{\smash[b]{n\sigma^2_e}})}.
\end{align*}
An $n$-dimensional ball $\mathbb{B}^n(r)$ with radius $r$ can be written in terms of the $n$-dimensional ball with radius 1: $\mathbb{B}^n(1)r^n$. Hence we find for the ratio of volumes $\text{vol}(X\beta)/\text{vol}(e)$
\begin{align*}
\sqrt{\det(XX^\top)}\left(\frac{||\beta||_2^2}{\sigma^2_e}\right)^{n/2}
=
\sqrt{\det(\alpha^{-1} XX^\top)},
\end{align*}
where $\alpha^{-1}=||\beta||_2^2/\sigma^2_e$ is the signal-to-noise-ratio (SNR).

It turns out that when $p>n$ we obtain $X^\top(XX^\top)Y$ as the so-called minimum norm estimate $\hat{\beta}^{MN}$ of $\beta$ \citep{BenIsraelGreville74,Bartlett:2021}. The minimum-norm solution is related to the ridge estimator as follows
\begin{align*}
\lim_{\alpha \to 0^+} (X^\top X + \alpha I)^{-1}X^\top = X^\top(XX^\top)^{-1}.
\end{align*}
Intuitively, the minimum-norm solution is obtained for the smallest $\alpha$ that makes the augmented inverse of $X^\top X$ possible. It is therefore reasonable to substitute for $X X^\top$ the ridge version $X^\top X+\alpha I$, where $\alpha$ is small. Hence, we obtain the volume ratio $\text{vol}((X^\top,\sqrt{\alpha}I)^\top\beta)/\text{vol}(e)$ with the ridge version \citep{Cheema:2020}
\begin{align*}
\sqrt{\det(\alpha^{-1} XX^\top+I_n)}=\sqrt{\det(X^\top X+\alpha I_p)},
\end{align*}
by the Weinstein–Aronszajn identity. 
\end{proof}
Because $\alpha=1/\text{SNR}=\sigma^2_e/||\beta||_2^2$, we obtain small $\alpha$, close to the minimum-norm estimate, if we have high signal $||\beta||_2^2$.

Clearly, the tuning parameter $\alpha$ can be expressed as the inverse of the signal-to-noise ratio (SNR). The SNR is defined as the ratio of the Euclidean length of the parameters ($||\beta||_2^2$, see Appendix \ref{sec:appendix-estimation-p-n}) and the noise variance $\sigma_e^2$; then $\alpha=1/\text{SNR}$ and $\text{SNR}= ||\beta||_2^2/\sigma^2_e$. The Euclidean length is a representation of the volume for the variables $\beta$. So, if we restrict $||\beta||_2^2$ to be no larger than $c$ (ridge regression), then we are restricting the volume of the possible values that the parameters in $\beta$ can take. 
This implies that when we increase the volume ($||\beta||_2^2$) we reduce the tuning parameter $\alpha$, and hence penalise the regression less. This can be interpreted as follows: imposing a higher $\alpha$ in ridge regression will lead to decreased variance (this was shown in Section \ref{sec:mse-decomposition}). This also implies that the peak at the interpolation point in Figure \ref{fig:ridge-mse-var-bias} is an artifact of applying a suboptimal tuning parameter (see Figure \ref{fig:mse-var-bias}, where in the middle row the peak has disappeared upon careful selection of $\alpha$). 
Summarising, we can say the following.
\begin{itemize}
\item[($a$)] We can increase $\alpha$ directly, reducing the test variance, and hence obtaining reasonable test MSE. This leads indirectly to a lower SNR (because $\alpha=1/\text{SNR}$); or
\item[($b$)] we can constrain the size of $||\beta ||_2$ in the ridge estimator, therefore, decreasing the model complexity (volume) of the model. This will reduce SNR and hence, increase $\alpha$, leading to a decrease in the test variance. We do this by constraining the ridge estimator with small $c$ such that $||\beta||_2\le c$. 
\end{itemize}
In both cases ($a$) and ($b$) we are either directly or indirectly constraining the total signal $||\beta||_2$. In model selection this can be done by either increasing $\alpha$, or equivalently, decreasing $c$ in the ridge constraint $||\beta||_2\le c$, or by incorporating the complexity (volume) of the model (the $n$-dimensional ball $\mathbb{B}^n(||\beta||_2)$, see Appendix \ref{sec:appendix-geometry-metric-entropy}). 
\section{Akaike information criterion}\label{sec:appendix-akaike-information-criterion}\noindent
For the derivation of the Akaike information criterion (AIC), we follow \citet{Claeskens:2008}. The AIC can be derived from the Kullbeck-Leibler divergence (KL). The KL is a way to quantify the overlap between  two distributions. The KL is 0 if the distributions are the same (almost surely) and is $>0$ if they are not the same. Hence, the AIC aims to minimise the KL across the different models. 

Let $f$ and $g$ be two densities, where we designate $g$ as the true underlying distribution and we use $f$ as an approximation. In practice we estimate a parameter $\hat{\beta}$ that we plug-in $f$, denoted by $f(\cdot,\hat{\beta})$. The KL is for $f$ and $g$ defined as \citep{KullbackLeibler51,Cover:2006}
\begin{align*}
\text{KL}(f,g) = \mathbb{E}_{Y\mid x} \log g(Y\mid x) - \mathbb{E}_{Y\mid x} \log f(Y\mid x,\hat{\beta}),
\end{align*}
where $\mathbb{E}_{Y\mid x}$ indicates the expectation over $Y$ conditional on $x$ (the values of the predictors in the nodewise GGM). Across different models the term $\mathbb{E}_{Y\mid x} \log g(Y\mid x)$ is fixed and so we concentrate on the second term. So, maximising the second term of the KL equals minimising the KL itself. The expectation of this second term is 
\begin{align*}
L=\mathbb{E}_{X} \left( \mathbb{E}_{Y\mid x} \log f(Y\mid x,\hat{\beta}) \right).
\end{align*}
The AIC tries to estimate this expectation and then select the model with the highest value (and hence the minimal value of KL). This expectation is estimated by the log-likelihood $\tfrac{1}{n}L(\hat{\beta})=\frac{1}{n}\sum_{i=1}^n f(Y_i\mid x_i,\hat{\beta})$.
And it can then be shown that this estimate is biased and needs to be corrected, since
\begin{align*}
\mathbb{E} (\tfrac{1}{n}L(\hat{\beta}) - L) \approx \frac{p}{n}\quad \text{and, hence}\quad \frac{1}{n}(-2L(\hat{\beta}) + 2p).
\end{align*}
Discarding the $1/n$ leads to the standard AIC. 

Here, for the GGM we used the Gaussian distribution for $f$. Then \citep{Seber:2012}
\begin{align*}
\sum_{i=1}^n\log f(Y_i\mid x_i,\beta) = \sum_{i=1}^n-\frac{1}{2\sigma^2}(Y_i-x_i^\top \beta)^2 - \frac{n}{2}\log \sigma^2 -\frac{n}{2}\log 2\pi.
\end{align*}
Plugging in our ridge estimate $\hat{\beta}_J$ for model $J$ (see (\ref{eq:ridge-estimator})), multiplying by -2 and ignoring constants, we obtain the AIC
\begin{align*}
\text{AIC}(J) = n\log \hat{\sigma}_{J} + 2 p_{J},
\end{align*}
where $p_{J}$ is the number of parameters for model $J$ and 
\begin{align*}
\hat{\sigma}_{J}^2=\frac{1}{n}\sum_{i=1}^n (y_i - x_i^\top \hat{\beta}_{J})^2
\end{align*}
is plugged in for $\sigma^2$. The standard AIC has a fixed (independent of the data) value of the ridge parameter $\alpha$. The AIC-CV has a ridge parameter $\alpha$ that has been obtained with $k$-fold cross validation.

\section{Minimum description length: optimised}\label{sec:appendix-minimum-description-length-optimised}\noindent
The minimum description length (MDL) principle has been extended to the high-dimensional setting of $n<p$ in another way than in Section \ref{sec:minimum-description-length}. The principle of approximating the code length for the data given the model and the code length for the model is the same, but in order to account for the high-dimensions (such that $d/n\to \gamma <\infty$), the effective degrees of freedom are used in combination with an optimal $\alpha$ for the effective degrees of freedom \citep{Dwivedi:2020}. The MDL is then obtained by minimising a function for $\alpha$ that gives the smallest MDL value. The function to be optimised over $\alpha$ is 
\begin{align*}
\text{MDL-opt} = \frac{n}{2}\log \hat{\sigma}^2 + \frac{1}{2\hat{\sigma}^2}||\hat{\beta}||_2^2 + \sum_{j=1}^{\min\{n,d\}} \log\left( 1 + \frac{\lambda_j}{\alpha} \right),
\end{align*}
where $\lambda_j$ are the eigenvalues of $X^\top X$ and $\hat{\beta}$ is the ridge estimator with value $\alpha$. The second term is included here because of the Bayesian interpretation of MDL in \citet{Dwivedi:2020}, where this term is entered as a prior for $\beta$. This version of MDL is referred to as MDL-opt.

\section{MDL leads to underfit in high-dimensional settings}\label{sec:appendix-mdl-underfit}\noindent
{\bf Proposition} \ref{prop:mdl-underfit}
Let $S$ denote the true support for a (possibly nonlinear) data generating process of $Y_i$ for $i=1,\ldots,n$ with $n$ fixed, and let $J$ represent the set of predictors such that $S\subset J$, where $|S|=d$ and $|J|=p$. We test linear models $X\beta_J$ for any $J$ (including $S$) by MDL as given in (\ref{eq:mdl-full}) with (\ref{eq:mdl-penalty}). Let $p$ increase and $p > d$. If $\mathbb{E}\hat{\sigma}^2_J>\gamma$, for some $\gamma>0$, then with probability at least $1-\tfrac{d}{p}\exp(-(p-d))$ it holds that $\hat{S}\subseteq S$; equivalently, the probability of false positives goes to 0. \\
\begin{proof}
We first note that 
\begin{align*}
\mathbb{P}(\hat{S}\subseteq S) = \mathbb{P}(\min_J\text{MDL}(J)\le \text{MDL}(S): |J|\ge |S|)
\end{align*}
And using the complement rule and exponentiation we can rewrite this as
\begin{align*}
1- \mathbb{P}(\min_J \exp[\text{MDL}(S) - \text{MDL}(J)] <1: |J|\ge |S|) 
\end{align*}
Then we apply Markov's inequality to show that for $d=|S|\le |J|=p$
\begin{align*}
\mathbb{P}( \exp[\text{MDL}(S)- \text{MDL}(J)] < 1) 
\le 
\mathbb{E}\left( \exp[\text{MDL}(S)- \text{MDL}(J)]  \right)
\end{align*}
We start with the first part of the MDL, the log of the empirical variance, defined as (using maximum likelihood) $\hat{\sigma}_J^2=\tfrac{1}{n}||Y-X^\top\hat{\beta}_J||^2$. For models with $p= |J|\ge |S|$, we have that the in-sample variance is $\hat{\sigma}^2_J\le \hat{\sigma}^2_S$ $\mathbb{P}$-a.e. \citep{BilodeauBrenner99}, and so $\mathbb{E}\hat{\sigma}^2_J\le \mathbb{E}\hat{\sigma}^2_S$ by monotonicity. Therefore, and because both terms are positive, using Jensen's inequality for concave functions such that $\mathbb{E}(\log X) \le \log \mathbb{E}(X)$, we obtain that 
\begin{align*}
0 \le \mathbb{E}\log \hat{\sigma}^2_S - \mathbb{E}\log \hat{\sigma}^2_J
\le
\log \mathbb{E}\hat{\sigma}^2_S - \log \mathbb{E}\hat{\sigma}^2_J
=
\log \frac{\mathbb{E}\hat{\sigma}^2_S} {\mathbb{E}\hat{\sigma}^2_J}
\end{align*}
The expectation $\mathbb{E}\tfrac{1}{n}||Y-X^\top\hat{\beta}_J||^2$ under the true model is $\sigma^2 + \mathbb{E}\tfrac{1}{n}||f(X,\beta_S)-X^\top\hat{\beta}_J||^2$, and similarly for the expectation $\sigma^2 + \mathbb{E}\tfrac{1}{n}||f(X,\beta_S)-X^\top\hat{\beta}_J||^2$, which are bounded by the assumption on misspecification in Lemma \ref{lem:mse-misspecification}. Using monotonicity above and the assumption that $\mathbb{E}\hat{\sigma}^2_J>0$ is bounded away from 0, the ratio above is positive and bounded by $K/\gamma$, with $|f(X,\beta_S)-X^\top\hat{\beta}_S|<K$. 

Then we move on to the next part $\text{tr}(S_J)\log n$ of MDL. In the setting of linear regression, we have that the sets of predictors in the design matrix $X=(x_1,x_2,\ldots,x_p)$ are nested, so that $S\subseteq J$ for $|J|\ge |S|$. For a fixed $\alpha>0$, $n>|S|=d$, and  $|J|=p>n$, we obtain 
\begin{align*}
\text{tr}(S_S)=\sum_{j=1}^{d} \frac{\lambda_j}{\lambda_j+\alpha} \le \sum_{j=1}^{n} \frac{\lambda_j}{\lambda_j+\alpha} =\text{tr}(S_J)\le \sum_{j=1}^{d} \frac{\lambda_j}{\lambda_j+\alpha}  + (n-d)\frac{1}{\alpha}
\end{align*}
since $\lambda_j\ge 0$. Then, for the difference between these two MDL parts we obtain
\begin{align*}
(\text{tr}(S_S)-\text{tr}(S_J))\log n \le -(n-d)\frac{1}{\alpha} \log n
\end{align*}

For the last part of MDL, we use a result by \citet[][Theorem 3]{Cheema:2020} 
\begin{align*}
\mathbb{E} \left(  \log \det(X^\top X+\alpha I) + \log \mathbb{B}^p(||\beta||_2)/\sigma^p  \right) \le \log d + \frac{n}{2}\log(\alpha+1) + \log\mathbb{B}^p(1)
\end{align*}
Because $\mathbb{B}^p(1)\le \mathbb{B}^d(1)$, where $d=|S|\le |J|=p$, we obtain that 
\begin{align*}
\log \frac{\mathbb{B}^d(1)}{\mathbb{B}^p(1)} = \log \pi^{-(p-d)/2}\exp(-(p-d)/2) +O(d/p)
\end{align*}
where we used Stirling's approximation for factorials and the approximation of the Euclidean ball $\mathbb{B}^p(1)$ given in (\ref{eq:ball-volume}). 

Putting the three elements together gives the upper bound on the expectation
\begin{align*}
&\log \frac{\mathbb{E}\hat{\sigma}^2_S} {\mathbb{E}\hat{\sigma}^2_J}  + \log n^{-(n-d)/\alpha} + \log\frac{d}{p} + \log \pi^{-(p-d)/2}\exp(-(p-d)/2)
\end{align*}
Since the first term is bounded by $K/\gamma$, we observe that the above is dominated by the last term, which goes to $0$ as $p$ increases. In particular, we find that 
\begin{align*}
\pi^{-(p-d)/2}\exp(-(p-d)/2) \le \frac{d}{p}\exp(-(p-d))
\end{align*}
Consequently, we obtain that 
\begin{align*}
\mathbb{E}\left( \exp[\text{MDL}(S)- \text{MDL}(J)]  \right) \to 0
\end{align*}
And, therefore, as $p$ increases
\begin{align*}
\mathbb{P}(\hat{S}\subseteq S) = \mathbb{P}(\min_J\text{MDL}(J)\le \text{MDL}(S): |J|\ge |S|) \le 1 - \frac{d}{p}\exp(-(p-d))
\end{align*}
as was to be shown.
\end{proof}

\section{Simulation details}\label{sec:appendix-simulation-details}\noindent
In the case where the true model is linear we have $f(X)=X\beta$. In the misspecified case we used instead of the linear model the sigmoid function $f(X)= 1/(1+\exp(X\beta))$. Again we keep $d=5$ non-zero coefficients fixed. The coefficients $\beta$ were normalised so that the $L_2$ norm was 1, i.e., $||\beta||_2=1$. 
The signal-to-noise ratio (SNR) is defined as $||\beta||_2^2/\sigma^2$, the ratio of the variance of the signal (coefficients) and the noise variance. Since we fix $||\beta||_2^2=1$, the variance of the noise determines the SNR, set to 2, so that the noise variance is 0.5. The number of generated observations is fixed at $n=40$.

For estimation we selected a fraction of 0.80 for training (estimation) of parameters and the remaining 0.20 for testing. We fix $n=40$ observations in total, and so obtain $n_\text{train}=32$ and $n_\text{test}=8$. We fix the dimension of the true model at $d=5$ and vary the number of parameters $p$ from 3 to 320. Hence we obtain the ratio $p/n$ for training data from 0.09 to 10, where the last setting means we have ten parameters per observation. With fixed $d=5$ we have that the correct ratio $p/n$ is at $5/32= 0.15625$. 

The ridge estimate is either obtained with parameter $\alpha=0.5$, or the ridge estimate is obtained with $\alpha$ estimated with 10-fold cross-validation, with the $\alpha$ with the smallest MSE. The Lasso is obtained with parameter $\alpha$ from 10-fold cross-validation. Estimation was performed with the \texttt{R} package \texttt{glmnet}.

\begin{figure}
\begin{tabular}{@{\hspace{-1em}} c @{\hspace{1em}} c @{\hspace{-1em}} c}
\textsf{}		&\textsf{test MSE}			&\textsf{bias-variance}\\[-2.5em]
\raisebox{8em}{\textsf{ridge}}		&\pgfimage[width=.5\textwidth]{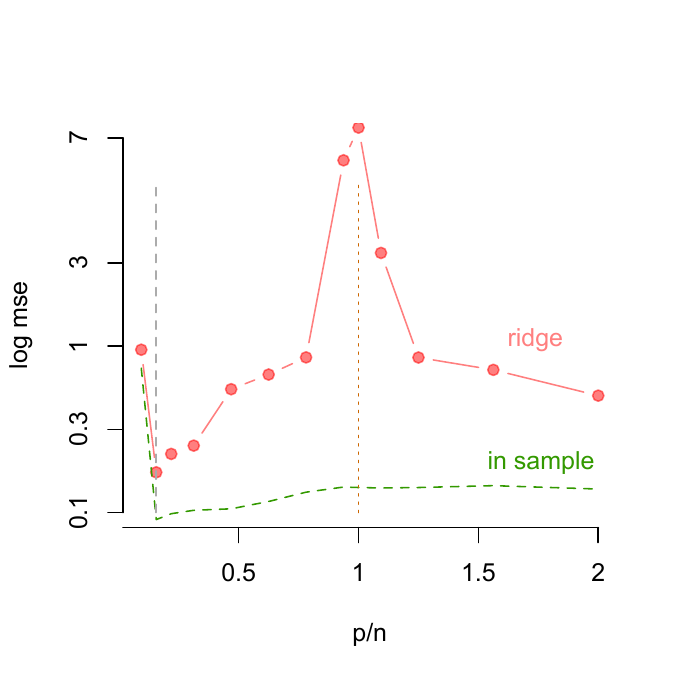}		&\pgfimage[width=.47\textwidth]{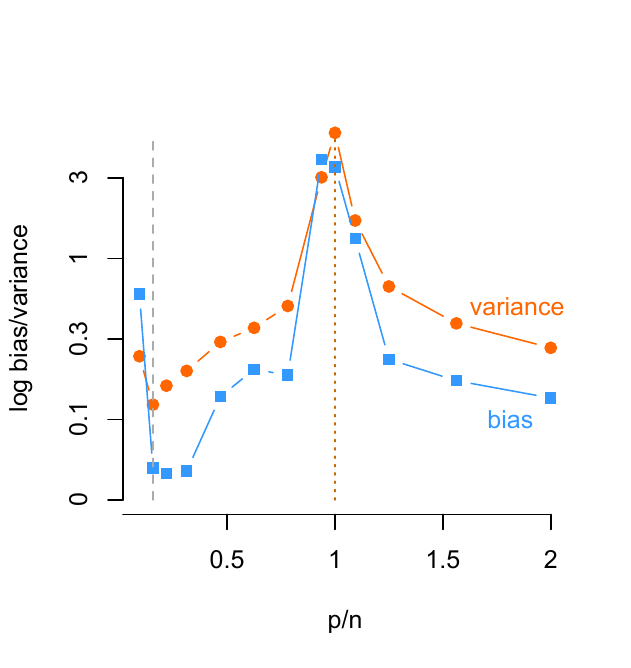}\\[-3em]
\raisebox{8em}{\textsf{ridge cv}}		&\pgfimage[width=.5\textwidth]{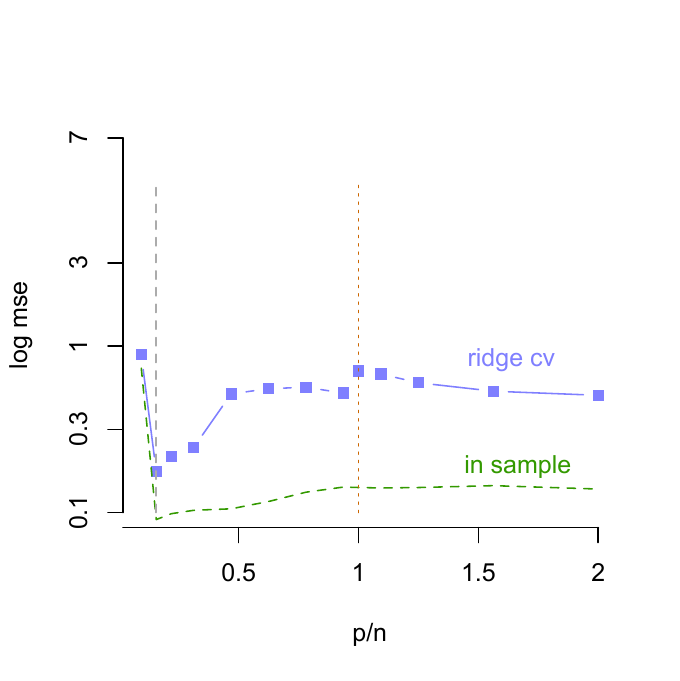}		&\pgfimage[width=.5\textwidth]{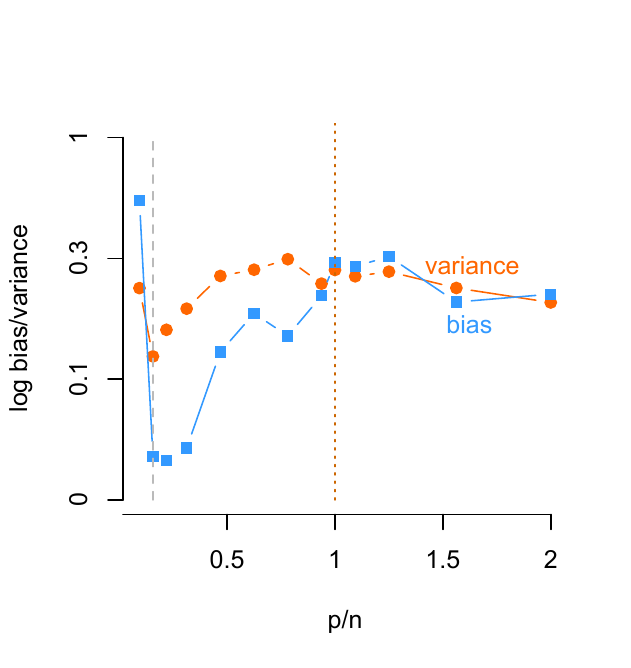}\\[-3em]
\raisebox{8em}{\textsf{lasso cv}}		&\pgfimage[width=.5\textwidth]{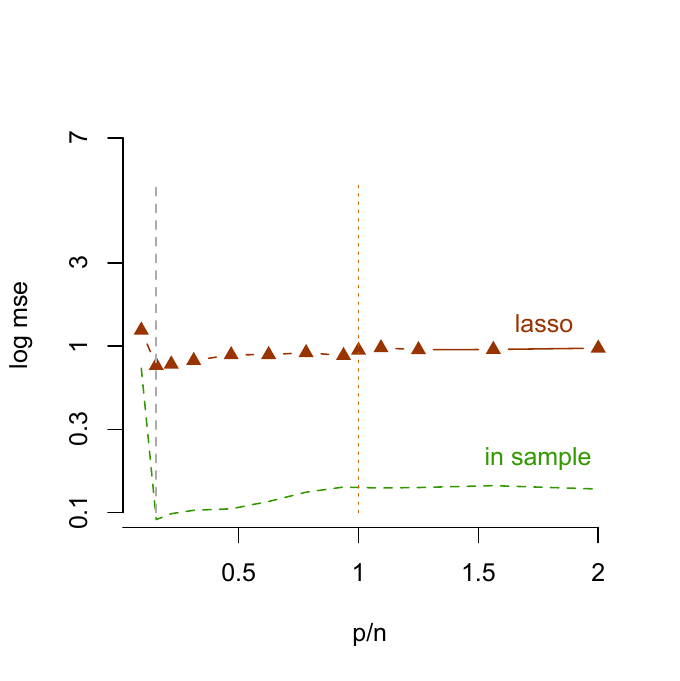}		&\pgfimage[width=.5\textwidth]{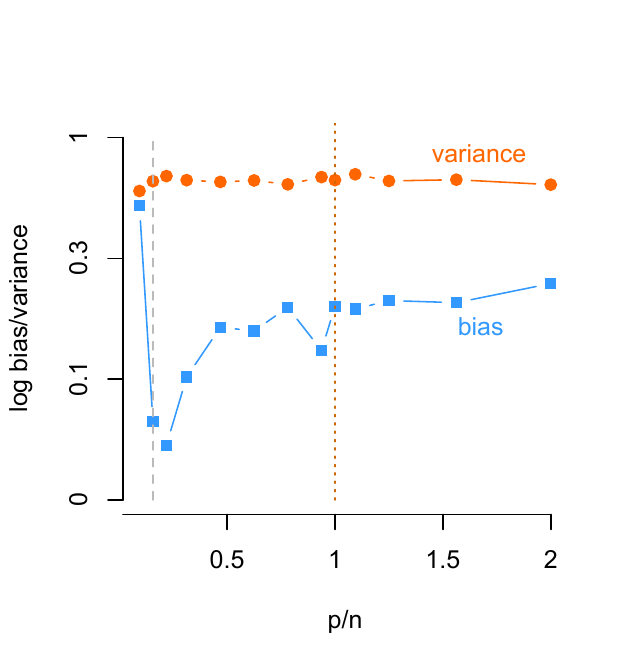}\\
\end{tabular}
\caption{Test MSE (left column) and its decomposition in squared bias and variance. The top row is obtained with a fixed value $\alpha=0.0001$, the middle row is obtained with an $\alpha$ estimated using $10$-fold cross-validation, and the bottom row is obtained with the lasso, where also the penalty parameter is obtained with 10-fold cross-validation. }
\label{fig:mse-var-bias}
\end{figure}
%


\section{Residuals in training and test samples}\label{sec:app-residuals-training-test}\noindent
The interpolation point at $p=n$ is the point where it is possible that each datapoint is captured by the model in the training set, also in the linear case. We see this in Figure \ref{fig:residuals}(a), where the residuals of the training (estimation) data as a function of the predicted values $\hat{Y}$ are plotted. Here, the true number of non-zero coefficients is 5 but in the overparameterised case we use $p=64$. As seen, the training residuals are 0 and so $\hat{Y}_{S}=\sum_{j\in S} X_j\beta$ is the same as $\hat{Y}_J = \sum_{j\in J}X_j \hat{\beta}_J$, where $\hat{\beta}_J$ is the ridge estimate for model $J$ and model $J$ contains all 64 predictors. Because the observations are predicted exactly, it was assumed that the variance of the residuals on a test set would be large. Figure \ref{fig:residuals}(b) shows that this is not the case. There is some variance in the residuals of the test set but not too much. The noise that ends up in the predictors is ``hidden'' in the unimportant directions of the predictors because of the ridge parameter \citep{Bartlett:2020}. 
\begin{figure}
\begin{tabular}{@{\hspace{-1em}} c @{\hspace{1em}} c @{\hspace{-1em}} c}
&\pgfimage[width=.5\textwidth]{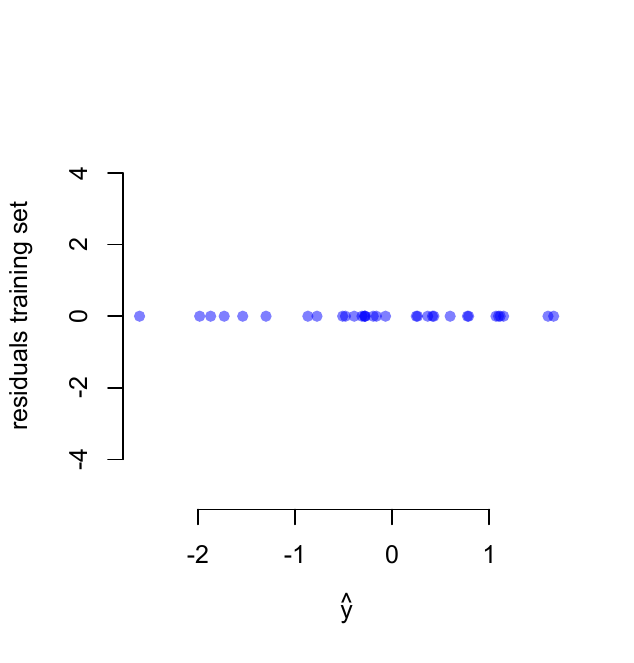}		&\pgfimage[width=.5\textwidth]{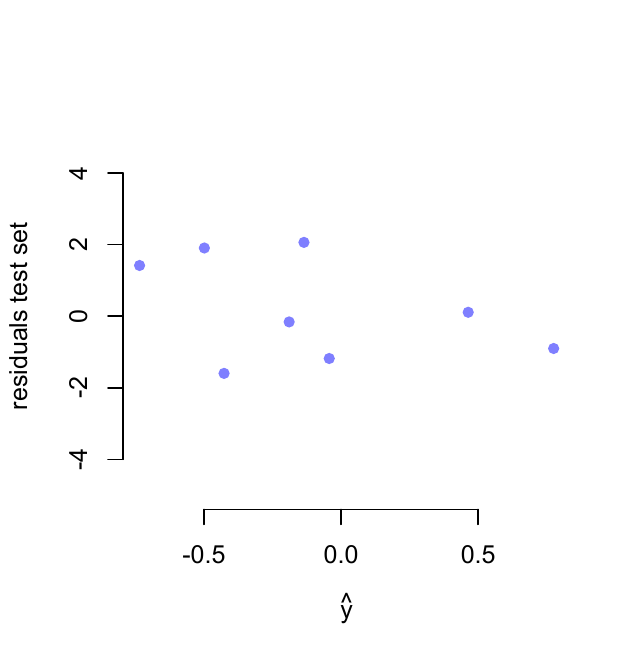}\\[-1em]
	&(a)		&(b)
\end{tabular}
\caption{Residuals of the training set (a), showing that the model with 64 predictors (while correct is 5) is interpolating the data, fitting exactly each point. The residuals of the test set (b) remain relatively small, indicating that the variance need not explode beyond the interpolation point.}
\label{fig:residuals}
\end{figure}
\end{document}